\documentclass[onefignum,onetabnum]{siamart190516}

\usepackage[utf8]{inputenc} 
\usepackage[T1]{fontenc}    
\usepackage{hyperref}       
\usepackage{url}            
\usepackage{booktabs}       
\usepackage{amsfonts}       
\usepackage{nicefrac}       
\usepackage{microtype}      
\usepackage{xcolor}         

\usepackage{lipsum}
\usepackage{graphicx}
\usepackage{epstopdf}
\usepackage{mathtools}        
\usepackage[noend]{algorithmic}
\usepackage{bm}

\newsiamremark{remark}{Remark}
\newsiamremark{hypothesis}{Hypothesis}
\crefname{hypothesis}{Hypothesis}{Hypotheses}
\newsiamthm{claim}{Claim}
\newtheorem{assumption}[theorem]{Assumption}

\newcommand\xbm{{\ensuremath{\bm{x}}}}
\newcommand\kbm{{\ensuremath{\bm{k}}}}
\newcommand\Kbm{{\ensuremath{\bm{K}}}}
\newcommand\fbm{{\ensuremath{\bm{f}}}}

\newcommand\ybm{{\ensuremath{\bm{y}}}}

\newcommand\Ibm{{\ensuremath{\bm{I}}}}

\title{On the convergence rate of noisy Bayesian Optimization with Expected Improvement}

\author{Jingyi Wang\thanks{Center for Applied Scientific Computing, Lawrence Livermore National Laboratory,
Livermore, CA 
  (\email{wang125@llnl.gov}).}
\and Haowei Wang \thanks{National University of Singapore (\email{haowei\_wang@u.nus.edu)}}
\and Nai-Yuan Chiang\footnotemark[1]
\and Cosmin G. Petra\footnotemark[1]
}
\usepackage{amsopn}

\usepackage{lmodern} 
\begin{document}
\newcommand{\Rbb}{\ensuremath{\mathbb{R} }}
\newcommand{\Pbb}{\ensuremath{\mathbb{P} }}
\newcommand{\Cbb}{\ensuremath{\mathbb{C} }}
\newcommand{\Ebb}{\ensuremath{\mathbb{E} }}
\newcommand{\Vbb}{\ensuremath{\mathbb{V} }}
\newcommand{\Nbb}{\ensuremath{\mathbb{N} }}
\newcommand{\norm}[1]{\left\lVert {#1} \right\rVert}
\newcommand\epsbold{{\ensuremath{\boldsymbol{\epsilon}}}}

\maketitle

\begin{abstract}
  Expected improvement (EI) is one of the most widely used acquisition functions in Bayesian optimization (BO). 
  Despite its proven success in applications for decades, important open questions remain on the theoretical convergence behaviors and rates for EI. In this paper, we contribute to the convergence theory of EI in three novel and critical areas.
  First, we consider objective functions that fit under the Gaussian process (GP) prior assumption, whereas existing works mostly focus on functions in the reproducing kernel Hilbert space (RKHS). 
  Second, we establish for the first time the asymptotic error bound and its corresponding rate for GP-EI with noisy observations under the GP prior assumption.
  Third, by investigating the exploration and exploitation properties of the non-convex EI function, we establish improved error bounds of GP-EI for both the noise-free and noisy cases. 
\end{abstract}


\section{Introduction}
Bayesian optimization (BO) is a derivative-free optimization method for black-box functions~\cite{lizotte2008,jones2001taxonomy,bosurvey2023} that enjoys enormous success in many applications such as hyperparameter optimization in machine learning~\cite{wu2019hyperparameter}, structural design~\cite{mathern2021}, robotics~\cite{calandra2016}, additive manufacturing~\cite{wang2023optimization}, inertial confinement fusion design~\cite{wang2023multifidelity}, etc. In recent years, BO has been increasingly applied to novel area such as expensive integral objectives~\cite{frazier2022}, distributionally robust optimization~\cite{shapiro2023}, risk-averse optimization~\cite{cakmak2020bayesian}, etc.
In the classic form, BO aims to solve the optimization problem 
\begin{equation} \label{eqn:opt-prob}
 \centering
  \begin{aligned}
	  &\underset{\substack{\xbm}\in C}{\text{minimize}} 
	  & & f(\xbm), \\
  \end{aligned}
\end{equation}
where $\xbm$ is the decision variable, compact set $C\subset \Rbb^d$, and $f:\Rbb^d\to \Rbb$ is the objective function. 

A surrogate model, often a Gaussian process (GP), is used to approximate the black-box objective function \eqref{eqn:opt-prob} ~\cite{frazier2018}. An acquisition function guides the selection of sequential points for observations. One of the most successful and widely used acquisition functions is the expected improvement (EI)~\cite{schonlau1998global}.
EI calculates the conditional expectation of an improvement function such that both the mean value and variance of the GP are used in search of the next sample. 
With a closed form, it is simple to implement in that only the cumulative distribution function (CDF) and probability density function (PDF) of the standard normal distribution are required. 
Thanks to its effectiveness and efficiency, EI has witnessed rich extensions in literature, \textit{e.g.}, 
scalable methods~\cite{eriksson2021scalable}, constrained EI~\cite{gardner2014}, etc. 
The classic noise-free BO algorithm using EI is also referred to as the efficient global optimization (EGO) algorithm~\cite{jones1998efficient}.

Despite its wide adoption and success, existing works on the theoretical convergence of Bayesian optimization with expected improvement, hereby referred to as GP-EI throughout the paper, have clear gaps. 
The two main streams of research study the asymptotic convergence and the cumulative regret bound of BO algorithms, the latter focusing on the 	performance over the entire optimization path. 
The asymptotic convergence analysis of GP-EI is limited by the assumptions on the objective and whether observation noise is allowed.
Under the assumptions that $f$ is in the RKHS of the GP kernel, also called the frequentist setting, with no observation noise,
the asymptotic convergence rates of GP-EI is established in~\cite{bull2011convergence}, where the convergence rates differ based on the kernels used in the GP. 
Specifically, define the error $r_t^0:=f(\xbm_t^+)-f^*$, where $\xbm_t^+$ is the sample that produces the best current observation among $t$ samples and $f^*=\underset{\substack{\xbm\in C}}{\text{argmin}} f(\xbm)$.
The error bound is $\mathcal{O}(t^{-\frac{1}{d}})$ for squared exponential (SE) kernels and $\mathcal{O}(t^{-\frac{\min\{\nu,1\}}{d}}\log^{\alpha}(t))$ for Matérn kernels, where $\nu$ is the parameter of the Matérn kernel and $\alpha=\frac{1}{2}$ for $\nu\in\Nbb$ and $\alpha=0$ otherwise. 
However, to the best of our knowledge, the extensions of~\cite{bull2011convergence} to \textit{noisy} GP-EI or when $f$ follows different assumptions, \textit{e.g.}, $f$ is sampled from a GP, are yet to be established.

The other stream of research has focused on the cumulative regret bound, based on the seminal work of~\cite{srinivas2009gaussian}.
In~\cite{wang2014theoreybo}, the authors added additional steps and auxiliary hyper-parameters to modify EI. Further, they assumed bounded conditions in RKHS on the hyper-parameters to derive a cumulative regret bound.
In~\cite{nguyen17a}, the authors used a stopping criterion to help bound the prediction error. However, the analysis relies on a critical yet not established result that the standard deviation at the optimal point decreases at the same rate as the sampled points. In~\cite{hu2022adjustedeiregret}, the authors attempted to bound the standard deviation at the optimal point via structured initial sampling and achieved an improvement on the regret bound for SE kernels. Their method requires an infinite number of initial samples as the total number of samples increases.
In~\cite{tran2022regret}, the authors proved no-regret for their modified EI, which includes additional control parameters for EI that could become unbounded as $t$ increases.

In most literature mentioned above, $f$ is assumed to be in RKHS of the kernel (see Definition~\ref{def:rkhs}).
Notably however, in~\cite{srinivas2009gaussian,lederer2019uniform}, the authors analyzed functions $f$ that are sampled from GP priors, also called the Bayesian setting. 
While the RKHS assumption is critical and general, limitations exist in the functions it represents. 
Functions in RKHS and functions sampled from GP priors do not encompass each other~\cite{srinivas2009gaussian}.
In~\cite{narcowich2006sobolev}, the authors stated that the smoother the kernel, the smaller the RKHS could be for which the convergence theories apply. In cases such as infinite-dimensional Gaussian,~\cite{van2011information} calls RKHS small relative to the support of the GP prior it is based on. 
Additionally, in~\cite{lederer2019uniform}, the authors argued that the constants in RKHS are difficult to compute in practice. 
Meanwhile, with proper choices of kernels, a function sampled from a GP prior can learn continuous functions with arbitrary precision~\cite{steinwart2001influence,van2011information}.
In~\cite{bect2019supermartingale}, the authors claimed that sample paths of GP often do not satisfy the RKHS assumption. 
Given its wide adoption in literature, the error bound and convergence analysis of GP-EI where the objective function $f$ is sampled from a GP prior is therefore of utmost importance and interest. 

Other works related to asymptotic convergence of GP-EI include~\cite{vazquez2010convergence}, where the authors showed that  the sequence of iterates generated by GP-EI are dense
In~\cite{ryzhov2016convergence}, the author showed that for the ranking and selection problems, a modified EI is identical to optimal computing budget allocation (OCBA) algorithm, via the asymptotic sampling ratios.
In~\cite{bect2019supermartingale}, the author proved consistency of EI for noiseless GP with continuous sample paths.


 
In this paper, we address gaps in the asymptotic convergence analysis of GP-EI mentioned above. 
Our contributions can be summarized as follows. 
First, we prove for the first time the asymptotic convergence rate for noisy GP-EI under the GP prior assumption for $f$ (see Assumption~\ref{assp:gp}).
Given $\delta\in(0,1)$, we prove that the upper bound for the noisy error measure $r_t:=y^+_t-f^*$, where $y^+_t$ is the best current observation among $t$ samples, is 
$\mathcal{O} \left(t^{-\frac{1}{2}}\log(t)^{\frac{d+1}{2}}\right)$ for SE kernels and $\mathcal{O}(t^{\frac{-\nu}{2\nu+d}}\log^{\frac{\nu}{2\nu+d}}(t))$ for Matérn kernels, with probability greater than $1-\delta$. 
In the noisy GP-EI case, the convergence rate refers to the rate of decrease for $r_t$, or equivalently, the rate at which $y_t^+$ reduces to values no larger than $f^*$.
Second, we present the asymptotic convergence rates without noise under the GP prior assumption, thereby extending the results in~\cite{bull2011convergence}, which are obtained under the RKHS assumption. Notably, we achieve the same convergence rates as~\cite{bull2011convergence} when using the same kernels.
Third, we prove improved error bounds for GP-EI with or without noise under the GP prior assumption.
The improved error bounds stem from our novel quantification of the trade-off between exploitation and exploration of EI, the latter being reflected in the posterior standard deviation, at non-sample points. The emerging analysis technique is general and can be used to analyze other convergence measure, \textit{e.g.}, the cumulative regret of GP-EI. As an example, we prove a better error bound compared to~\cite{bull2011convergence}, under the RKHS assumption and noiseless case.

This paper is organized as follows. In Section~\ref{se:bo}, we describe the basis of the GP-EI algorithm and other background information. In Section~\ref{se:analysis}, we prove multiple properties of EI and the preliminary theoretical results under GP prior assumptions, which are the foundation of the  analysis in Section~\ref{se:convergence}.
We present our asymptotic convergence analysis on noisy GP-EI under GP prior assumptions in Section~\ref{se:convergence}.
In Section~\ref{se:rkhs}, we improve asymptotic convergence results to noiseless GP-EI under RKHS assumptions.
Numerical illustrations of our theoretical developments are presented in Section~\ref{se:examples}, while conclusions are offered in Section~\ref{se:conclusion}.

\section{Bayesian optimization and expected improvement}\label{se:bo}
There are two main components of Bayesian optimization algorithms: a GP surrogate model for the objective function $f$, and an acquisition function that measures the benefit of the next sample and guides the sequential search.
We introduce both in the subsequent subsections.
\subsection{Gaussian process}\label{se:gp}
A GP takes the prior distribution of samples to be multivariate Gaussian distribution. Consider a zero mean GP with the kernel (\textit{i.e.}, covariance function) $k(\xbm,\xbm'):\Rbb^d\times\Rbb^d\to\Rbb$, denoted as $GP(0,k(\xbm,\xbm'))$. 
At each sample $\xbm_i\in C, i=1,2,\dots$, we consider observation noise $\epsilon_i$ for $f$ so that the observed function value is 
$y_i = f(\xbm_i)+\epsilon_i$.  
Thus, the prior distribution on 
$\xbm_{1:t}=[\xbm_1,\dots,\xbm_t]^T$, $\fbm_{1:t}=[f(\xbm_1),\dots,f(\xbm_t)]^T$ and $\ybm_{1:t}=[y_1,\dots,y_t]^T$
is $\fbm_{1:t}\sim\mathcal{N} (\mathbf{0}, \Kbm_t)$, 
where $\mathcal{N}$ denotes the normal distribution, 
and $\Kbm_t = [k(\xbm_1,\xbm_1),\dots,k(\xbm_1,\xbm_t); \dots; k(\xbm_t,\xbm_1),\dots,k(\xbm_t,\xbm_t)]$. 

We assume that the noise follows an independent zero mean Gaussian distribution with variance $\sigma^2$, \textit{i.e.}, $\epsilon_i \sim\mathcal{N}(0,\sigma^2), i=1,\dots,t$. 
The posterior distribution of  $f(\xbm) | \xbm_{1:t},\ybm_{1:t} \sim \mathcal{N} (\mu(\xbm),\sigma^2(\xbm))$ can be inferred using Bayes' rule. 
The posterior mean $\mu_t$ and variance  $\sigma^2_t$ are  
\begin{equation} \label{eqn:GP-post}
 \centering
  \begin{aligned}
  &\mu_t(\xbm)\ =\ \kbm_t(\xbm) \left(\Kbm_t+ \sigma^2 \Ibm\right)^{-1} \ybm_{1:t} \\
  &\sigma^2_t(\xbm)\ =\
k(\xbm,\xbm)-\kbm_t(\xbm)^T \left(\Kbm_t+\sigma^2 \Ibm\right)^{-1}\kbm_t(\xbm)\ ,
\end{aligned}
\end{equation}
where $\kbm_t(\xbm)= [k(\xbm_1,\xbm),\dots,k(\xbm_t,\xbm)]^T$.

We make standard assumptions that $k(\xbm,\xbm')\leq 1$ and $k(\xbm,\xbm)=1$, $\forall \xbm,\xbm'\in C$.
Choices of the kernels include the SE function, Matérn functions, etc~\cite{frazier2018}.
The improvement function of $f$ given $t$ samples is defined as 
\begin{equation} \label{eqn:improvement}
 \centering
  \begin{aligned}
    I_t(\xbm) = \max\{ y^+_{t} -f(\xbm),0  \}, 
  \end{aligned}
\end{equation}
where $y^+_{t}= \underset{\substack{y_i\in \ybm_{1:t}}}{\text{argmin}} \ y_i$ is the best existing observed objective.

\subsection{Expected improvement}
In this paper, we consider the probability space $(\Omega, \mathcal{F}, \Pbb)$, where $\Omega$ is the sample
space, $\mathcal{F}$ the $\sigma$-algebra generated by the subspace of $\Omega$ and $\Pbb$ the probability measure on $\mathcal{F}$. 
We denote a filtration (ordered increasing $\sigma$-algebra) of $\mathcal{F}$ as $\mathcal{F}_t$, which is often made equivalent to the history until $t$ defined by: $\mathcal{H}_t=\{(\xbm_s,y_s):s=1,\dots,t\}$ in BO literature~\cite{chowdhury2017kernelized}.
The EI acquisition function is defined as the expectation of~\eqref{eqn:improvement} conditioned on $t$ samples and GP, with a closed form expression:
\begin{equation} \label{eqn:EI-1}
 \centering
  \begin{aligned}
       EI_t(\xbm) =   (y_{t}^+-\mu_{t}(\xbm))\Phi(z_{t}(\xbm))+\sigma_{t}(\xbm)\phi(z_{t}(\xbm)), 
  \end{aligned}
\end{equation}
where 
  $z_{t}(\xbm) = \frac{y^+_{t}- \mu_{t}(\xbm)}{\sigma_{t}(\xbm)}$.
The functions 
$\phi$ and~$\Phi$ are the PDF and CDF of the standard normal distribution, respectively.

To further analyze the properties of EI, we distinguish between its exploration term, \textit{e.g.}, $\sigma_t(\xbm)$ in~\eqref{eqn:EI-1}, and exploitation term, \textit{e.g.}, $y_t^+-\mu_t(\xbm)$ in~\eqref{eqn:EI-1}, and define $EI(a,b):\Rbb\times\Rbb\to\Rbb$ as
\begin{equation} \label{eqn:EI-ab}
 \centering
  \begin{aligned}
       EI(a,b) = a \Phi\left(\frac{a}{b}\right)+b \phi\left(\frac{a}{b}\right).  
  \end{aligned}
\end{equation}
From the assumptions on kernels in Section~\ref{se:gp}, the exploration parameter $b$ is constrained by $b\in(0,1]$. We view $a$ and $b$ as two independent variables.
For a given $\xbm$, if $a_t=y^+_{t}-\mu_{t}(\xbm)$ and $b_t=\sigma_{t}(\xbm)\in [0,1]$, $EI(a_t,b_t)=EI_t(\xbm)$. 

Another commonly used function in the analysis of EI is  $\tau:\Rbb\to\Rbb$ defined as 
\begin{equation} \label{def:tau}
 \centering
  \begin{aligned}
   \tau(z) = z\Phi(z) + \phi(z). 
  \end{aligned}
\end{equation}
Thus, EI can be written as $EI_t(\xbm) = \sigma_t(\xbm) \tau(z_t(\xbm))$.
The next sample $\xbm_{t+1}$ in the GP-EI algorithm is chosen by maximizing the acquisition function over $C$, \textit{i.e.},  
\begin{equation} \label{eqn:acquisition-1}
 \centering
  \begin{aligned}
      \xbm_{t+1} = \underset{\substack{x\in C}}{\text{argmax}}  EI_t (\xbm).
  \end{aligned}
\end{equation}
In order to solve~\eqref{eqn:acquisition-1}, optimization algorithms such as L-BFGS or random 
search can be used.
The GP-EI algorithm is given in Algorithm~\ref{alg:boei}, where a stopping criterion can be used, \textit{e.g.}, a prescribed computational budget.
\begin{algorithm}[H]
 \caption{GP-EI algorithm}\label{alg:boei}
  \begin{algorithmic}[1]
	  \STATE{Choose $\mu_0=0$, $k(\cdot,\cdot)$, $\alpha$, and $T_0$ initial samples $\xbm_i, i=0,\dots,T_0$. Observe $y_i$.} 
	  \STATE{Train the Gaussian process surrogate model for $f$ on the initial samples.}
  \FOR{$t=T_0,T_0+1,\dots$}
	  \STATE{Find $\xbm_{t+1}$ based on~\eqref{eqn:acquisition-1}.}
	  \STATE{Observe $y_{t+1}=f(\xbm_{t+1})+\epsilon_{t+1}$. \;}
	  \STATE{Train the surrogate model with the addition of $\xbm_{t+1}$ and $y_{t+1}$.\;}
	  \IF {Stopping criterion satisfied} 
              \STATE{Exit}
	  \ENDIF
  \ENDFOR
  \end{algorithmic}
\end{algorithm}
\subsection{Additional background}
We present the remaining necessary background information and our assumptions in this section. 
Our analysis uses repeatedly union bound, or Boole's inequality, which we provide in the following lemma for completeness.
\begin{lemma}\label{lem:unionbound}
  For a countable set of events $A_1,A_2,\dots$, we have 
 \begin{equation*} \label{eqn:union-bound-1}
  \centering
  \begin{aligned}
    \Pbb(\bigcup_{i=1}^{\infty} A_i) \leq \sum_{i=1}^{\infty} \Pbb(A_i).
  \end{aligned}
\end{equation*}
\end{lemma}
In order to find the rate of decrease of the error bound in the noisy case, we use the following well-established lemma for GP based on information theory~\cite{srinivas2009gaussian}. 
\begin{lemma}\label{lem:variancebound}
 The sum of GP posterior variance $\sigma_{t}$ at next sample $\xbm_{t+1}$ satisfies
 \begin{equation} \label{eqn:var-1}
  \centering
  \begin{aligned}
    \sum_{i=1}^t  \sigma_{i-1}^2(\xbm_i) \leq C_{\gamma} \gamma_t, 
  \end{aligned}
\end{equation}
 where $C_{\gamma} = \frac{2}{log(1+\sigma^{-2})}$ and $\gamma_t$ is the maximum information gain after $t$ samples.
\end{lemma}
Information gain measures the informativeness of a set of sampling points in $C$ about $f$. Readers are referred to~\cite{cover1999elements,srinivas2009gaussian} for a detailed definition of the maximum information gain $\gamma_t>0$. 
The rate of increase for $\gamma_t$ is dependent on the property of the kernel.
For common kernels such as the SE kernel and the Matérn kernel, $\gamma_t$ and its order of increase have been widely studied in literature~\cite{vakili2021information}.
The state-of-the-art rates of $\gamma_t$ for two commonly-used kernels are summarized below~\cite{vakili2021information}.
\begin{lemma}\label{lem:gammarate}
  For a GP with $t$ samples, the SE kernel has $\gamma_t=\mathcal{O}(\log^{d+1}(t))$, and 
  the Matérn kernel with smoothness parameter $\nu>0$ has $\gamma_t=\mathcal{O}(t^{\frac{d}{2\nu+d}}(\log^{\frac{2\nu}{2\nu+d}} (t)))$.
\end{lemma}
The formal definition of RKHS is given below.
\begin{definition}\label{def:rkhs}
   Consider a positive definite kernel $k:C\times C\to\Rbb $ with respect to a finite Borel measure supported on $C$. A Hilbert space $H_k$ of functions on $C$ with an inner product $\langle \cdot,\cdot \rangle_{H_k}$ is called a RKHS with kernel $k$ if $k(\cdot,\xbm)\in H_k$ for all $\xbm\in C$, and $\langle f,k(\cdot,\xbm)\rangle_{H_k}=f(\xbm)$ for all $\xbm\in C, f\in H_k$. The induced RKHS norm $\norm{f}_{H_k}=\sqrt{\langle f,f\rangle_{H_k}}$ measures the smoothness of $f$ with respect to $k$.
\end{definition}
For ease of reference, we list the assumptions used in our analysis below, all of which are common in  literature.
First, we have the GP prior assumption for $f$, also called the Bayesian setting. 
We further assume that $f$ is Lipschitz continuous, as in~\cite{srinivas2009gaussian}.
\begin{assumption}\label{assp:gp}
The objective function $f$ is a sample from the Gaussian process $GP(0,k(\xbm,\xbm'))$, where $k(\xbm,\xbm')\leq 1$ and $k(\xbm,\xbm)=1$. 
Further, $f$ is Lipschitz continuous with constant $L$.
\end{assumption}
The RKHS assumption of $f$ is given below, also called the frequentist setting.
\begin{assumption}\label{assp:rkhs}
The objective function $f$ is in the RKHS $\mathcal{H}_k$ associated with the kernel $k(\xbm,\xbm')\leq 1$ and $k(\xbm,\xbm)=1$. Further, the RKHS norm of $f$ is bounded $\norm{f}_{\mathcal{H}_k}\leq B$ for some $B\geq 1$. 
Moreover, $f$ is Lipschitz continuous with constant $L$.
\end{assumption}
Without losing generality, the set constraint assumption is given below.
\begin{assumption}\label{assp:constraint}
   The set $C\subseteq [0,r]^d$ is compact ($r>0$).
\end{assumption}
The Gaussian assumption on the noise is given below.
\begin{assumption}\label{assp:gaussiannoise}
  The observation noise are i.i.d. random samples from a zero mean Gaussian, \textit{i.e.}, $\epsilon_t\sim\mathcal{N}(0,\sigma^2),\sigma>0$ for all $t$.
\end{assumption}

\section{Preliminary results}\label{se:analysis}
We present the properties of the non-convex EI function in Section~\ref{se:EIproperty} and the preliminary lemmas for GP-EI under the GP prior assumption in Section~\ref{se:preliminaryGP}, which are both important for the convergence analysis in Section~\ref{se:convergence}. 
We use union bound repeatedly to obtain the probability of multiple events or inequalities occurring. 
Without losing generality, we consider the case $\sigma_t(\xbm) >0$, as the results for the case $\sigma_t(\xbm)=0$ hold trivially. 

\subsection{EI properties}\label{se:EIproperty}
First, we state some basic properties of $\phi$, $\Phi$, and $\tau$ in the following lemma.
\begin{lemma}\label{lem:phi}
The PDF and CDF of standard normal distribution satisfy $0< \phi(x)\leq \phi(0), \Phi(x)\in(0,1)$, 
for any $x\in\Rbb$. 
 Given a random variable $\xi$ sampled from the standard normal distribution, \textit{i.e.}, $\xi\sim\mathcal{N}(0,1)$, we have $\Pbb\{\xi>c|c>0\}\leq \frac{1}{2}e^{-c^2/2}$. 
 Similarly, for $c<0$, $\Pbb\{\xi<c|c<0\}\leq \frac{1}{2}e^{-c^2/2}$. 
\end{lemma}
The last statement in Lemma~\ref{lem:phi} is a well-known result (\textit{e.g.}, see proof of Lemma 5.1 in~\cite{srinivas2009gaussian}). Figure~\ref{fig:tauandphi} (left) also illustrates this point.
The property of $\tau(\cdot)$ in~\eqref{def:tau} is given below.
\begin{lemma}\label{lem:tau}
  The function $\tau(\cdot)$ is monotonically increasing  and $\tau(z)>0$ for $\forall z\in \Rbb$.
  Moreover, $\frac{d \tau(z)}{d z}=\Phi(z)$.
\end{lemma}
\begin{proof}
   From the definition~\eqref{def:tau} of $\tau(z)$, we can write
  \begin{equation*} \label{eqn:tau-1}
  \centering
   \begin{aligned}
   \tau(z)=z\Phi(z)+\phi(z) > \int_{-\infty}^{z} u\phi(u) du + \phi(z) = -\phi(u)|_{-\infty}^z+\phi(z) =0.
  \end{aligned}
  \end{equation*}
  Given the definition of $\phi(u)$, 
  \begin{equation} \label{eqn:tau-2}
 \centering
  \begin{aligned}
      \frac{d \phi(u)}{d u} = \frac{1}{\sqrt{2\pi}} e^{-\frac{u^2}{2}} (-u) = -\phi(u)u.
   \end{aligned}
\end{equation}
   Thus, direct differentiation of $\tau(z)$ gives 
   \begin{equation*} \label{eqn:tau-3}
 \centering
  \begin{aligned}
      \frac{d\tau(z)}{d z} = \Phi(z)+z\phi(z) -\phi(z)z = \Phi(z)>0.
   \end{aligned}
\end{equation*}
\end{proof}
Another lemma for $\tau(\cdot)$ is given below, with an illustrative plot in Figure~\ref{fig:tauandphi} (right).
\begin{lemma}\label{lem:tauvsPhi}
  Given $z>0$, $\Phi(-z)>\tau(-z)$. 
\end{lemma}
\begin{proof}
   Define $q(z) = \Phi(-z)-\tau(-z)$.
   Using integration by parts, we have
   \begin{equation} \label{eqn:tvp-pf-1}
 \centering
  \begin{aligned}
   \Phi(z) = \int_{-\infty}^z \phi(u)du > \int_{-\infty}^{z}\phi(u)(1-\frac{3}{u^4})du=-\frac{\phi(z)}{z}+\frac{\phi(z)}{z^3}. 
   \end{aligned}
\end{equation}
  Replacing $z$ with $-z$ in~\eqref{eqn:tvp-pf-1},
   \begin{equation} \label{eqn:tvp-pf-1.5}
 \centering
  \begin{aligned}
   \phi(-z)\left(\frac{1}{z}-\frac{1}{z^3}\right)<\Phi(-z).
   \end{aligned}
\end{equation}
   Multiplying both sides in~\eqref{eqn:tvp-pf-1.5} by $1+z$,
\begin{equation} \label{eqn:tvp-pf-2}
 \centering
  \begin{aligned}
   (1+z)\Phi(-z)> \phi(-z)\frac{z^2-1}{z^3}(1+z)= \phi(-z)\left(1+\frac{z^2-z-1}{z^3}\right). 
   \end{aligned}
\end{equation}
   Thus, if $z > \frac{1+\sqrt{5}}{2}$, then the right-hand side of~\eqref{eqn:tvp-pf-2} is greater than $\phi(-z)$ and  
 \begin{equation} \label{eqn:tvp-pf-3}
 \centering
  \begin{aligned}
   q(z)=\Phi(-z)-\tau(-z) = (1+z)\Phi(-z) - \phi(-z)>0.
   \end{aligned}
\end{equation}
   Consequently, we focus on $z\in(0,\frac{1+\sqrt{5}}{2}]$. 
   Taking the derivative of $q(z)$, by Lemma~\ref{lem:tau},  
   \begin{equation} \label{eqn:tvp-pf-4}
 \centering
  \begin{aligned}
     q'(z)= \frac{d q(z)}{d z} = -\phi(-z)+\Phi(-z).
   \end{aligned}
\end{equation}
  Further, the derivative of $q'(z)$ is 
  \begin{equation} \label{eqn:tvp-pf-5}
 \centering
  \begin{aligned}
      \frac{d^2 q(z)}{d z^2} = \frac{d q'(z)}{d z} = -\phi(-z)+\phi(-z) z=\phi(z)(z-1).
   \end{aligned}
\end{equation}
  For $z>1$, $\frac{d^2 q(z)}{d z^2}>0$. For $z\in (0,1)$ $\frac{d^2 q(z)}{d z^2}<0$. 
  Thus, $q'(z)$ is monotonically increasing for $z>1$ and decreasing for $z\in (0,1)$.
   For $z>1$, from Lemma~\ref{lem:tau} and definition of $\tau$,  
    \begin{equation} \label{eqn:tvp-pf-6}
 \centering
  \begin{aligned}
    q'(z) =  \Phi(-z)-\phi(-z) < z\Phi(-z) - \phi(-z) = - \tau(-z) < 0.
   \end{aligned}
\end{equation}
   Further, $q'(0)=\Phi(0)-\phi(0)>0$ and $q'(1)=\Phi(-1)-\phi(-1)<0$. Thus, there exists unique $\bar{z}\in (0,1)$ so that $q'(\bar{z})=0$.
   For $z\in (0,\bar{z})$, $q(z)$ is monotonically increasing. For $z\in [\bar{z},\frac{1+\sqrt{5}}{2}]$, $q(z)$ is monotonically decreasing. Therefore, $q(z) >\min\{q(0), q(\frac{1+\sqrt{5}}{2})\}$ for $z\in (0,\frac{1+\sqrt{5}}{2})$. Since $q(0)>0$ and $q(\frac{1+\sqrt{5}}{2})>0$, we have $q(z)>0$ for $z\in (0,\frac{1+\sqrt{5}}{2})$. 
   Combined with~\eqref{eqn:tvp-pf-3}, the proof is complete.
\end{proof}
The next lemma proves basic properties for $EI_t$.
\begin{lemma}\label{lem:EI}
 For $\forall \xbm\in C$, 
    $EI_t(\xbm) \geq 0$ and $EI_t(\xbm) \geq y^+_{t}- \mu_{t}(\xbm)$.
Moreover, 
\begin{equation} \label{eqn:EI-property-2}
 \centering
  \begin{aligned}
   z_{t}(\xbm)\leq  \frac{EI_{t}(\xbm)}{\sigma_{t}(\xbm) } < \begin{cases}  \phi(z_{t}(\xbm)), \ &z_{t}(\xbm)<0\\
             z_{t}(\xbm) +\phi(z_{t}(\xbm)), \ &z_{t}(\xbm)\geq 0.
      \end{cases}
  \end{aligned}
\end{equation}

\end{lemma}
\begin{proof}
     From the definition of $I_t$ and $EI_t$, the first statement follows immediately.
     By~\eqref{eqn:EI-1},  
\begin{equation} \label{eqn:EI-property-pf-1}
 \centering
  \begin{aligned}
     \frac{EI_t(\xbm)}{\sigma_{t}(\xbm) } = z_{t}(\xbm) \Phi(z_{t}(\xbm)) + \phi(z_{t}(\xbm)). 
  \end{aligned}
\end{equation}
If $z_{t}(\xbm)<0$, or equivalently $y^+_{t} -\mu_{t}(\xbm)<0$,~\eqref{eqn:EI-property-pf-1} leads to
     $\frac{EI_t(\xbm)}{\sigma_{t}(\xbm)} < \phi(z_{t}(\xbm))$. 
If $z_{t}(\xbm)\geq 0$, $\Phi(\cdot)<1$ gives us
     $\frac{EI_t(\xbm)}{\sigma_{t}(\xbm)} < z_{t}(\xbm) +\phi(z_{t}(\xbm))$. 
 The left inequality in~\eqref{eqn:EI-property-2} is a direct consequence of $EI_t(\xbm) \geq y^+_{t}- \mu_{t}(\xbm)$.
\end{proof}
The monotonicity of the exploration and exploitation of form~\eqref{eqn:EI-ab} is given next.
\begin{lemma}\label{lem:EI-ms}
  $EI(a,b)$ is monotonically increasing with respect to $a$ and $b$ for $b\in(0,1]$.
\end{lemma}
\begin{proof}
   Taking the derivative of $EI(a,b)$ with respect to $a$,  
 \begin{equation} \label{eqn:EI-ms-pf-1}
 \centering
  \begin{aligned}
   \frac{\partial EI(a,b)}{\partial a} = \Phi\left(\frac{a}{b}\right) + a\phi\left(\frac{a}{b}\right) \frac{1}{b} + b\frac{\partial \phi\left(\frac{a}{b}\right)}{\partial a}.
   \end{aligned}
\end{equation}
   From~\eqref{eqn:tau-2},~\eqref{eqn:EI-ms-pf-1} is  
  \begin{equation} \label{eqn:EI-ms-pf-2}
 \centering
  \begin{aligned}
   \frac{\partial EI(a,b)}{\partial a} = \Phi\left(\frac{a}{b}\right) + \phi\left(\frac{a}{b}\right) \frac{a}{b}  -\phi\left(\frac{a}{b}\right)\frac{a}{b}  = \Phi\left(\frac{a}{b}\right)>0.
   \end{aligned}
\end{equation}
  Similarly, differentiation of~\eqref{eqn:EI-ab} with~\eqref{eqn:tau-2} gives  
 \begin{equation} \label{eqn:EI-ms-pf-3}
 \centering
  \begin{aligned}
   \frac{\partial EI(a,b)}{\partial b} =&  -a\phi\left(\frac{a}{b}\right)\frac{a}{b^2} + \phi\left(\frac{a}{b}\right)-b\phi\left(\frac{a}{b}\right)\frac{a}{b}(-\frac{a}{b^2})
     = \phi\left(\frac{a}{b}\right)>0.
   \end{aligned}
\end{equation}
\end{proof}

\subsection{Characterizations of GP prior objective}\label{se:preliminaryGP}
The following lemma characterizes the CDF of $I_t(\cdot)$ under the GP prior assumption.
\begin{lemma}\label{lem:Icdf}
 Under Assumption~\ref{assp:gp}, the probability distribution of $I_t$ satisfies 
\begin{equation*} \label{eqn:Icdf-1}
 \centering
  \begin{aligned}
     \Pbb\{I_t(\xbm) \leq a\} = \begin{cases}
                    0,   \ &a<0,\\
                  \Phi\left(\frac{a}{\sigma_{t}(\xbm)}-z_{t}(\xbm)\right), \  &a\geq 0.
                  \end{cases}
  \end{aligned}
\end{equation*}
\end{lemma}
\begin{proof}
   Under Assumption~\ref{assp:gp}, at a given $t$, $f(\xbm)\sim\mathcal{N}(\mu_{t}(\xbm),\sigma_{t}(\xbm))$. Since $I_t(\xbm)\geq 0$ for all $\xbm$,~\eqref{eqn:Icdf-1} follows immediately if $a<0$.
   For $a\geq 0$,  
\begin{equation*} \label{eqn:Icdf-pf-1} 
 \centering
  \begin{aligned}
   \Pbb\{I_t(\xbm) \leq a\} = \Pbb\{y^+_{t}-f(\xbm) \leq a\} = 1- \Pbb\{f(\xbm) \leq y^+_{t}-a\}.
  \end{aligned}
\end{equation*}
  Using basic properties of the standard normal CDF, 
\begin{equation*} \label{eqn:Icdf-pf-2} 
 \centering
  \begin{aligned}
   1-\Pbb\{f(\xbm) \leq y^+_{t}-a\}=1-\Phi\left(\frac{y_t^+-a-\mu_t(\xbm)}{\sigma_t(\xbm)} \right)=\Phi\left(\frac{a-y_t^++\mu_t(\xbm)}{\sigma_t(\xbm)} \right).
  \end{aligned}
\end{equation*}
\end{proof}
The next lemma is a well-known result on the prediction error $f(\xbm)-\mu_t(\xbm)$ under Assumption~\ref{assp:gp}~\cite{srinivas2009gaussian}.
\begin{lemma}\label{lem:fmu}
Given $\delta\in(0,1)$, let  $\beta = 2\log(\frac{1}{\delta})$. Under Assumption~\ref{assp:gp}, for any given $\xbm\in C$ and $t\in\Nbb$,  the following holds 
\begin{equation} \label{eqn:fmu-1}
 \centering
  \begin{aligned}
     \Pbb\{|f(\xbm) - \mu_{t}(\xbm)  | \leq \sqrt{\beta} \sigma_{t}(\xbm)\} \geq 1-\delta. 
  \end{aligned}
\end{equation}
Moreover, the one-sided inequalities hold with probability greater than $1-\frac{\delta}{2}$, \textit{i.e.}, 
     $\Pbb\{f(\xbm) - \mu_{t}(\xbm)   \leq \sqrt{\beta} \sigma_{t}(\xbm)\}\geq 1-\frac{\delta}{2}$, 
and 
     $\Pbb\{f(\xbm) - \mu_{t}(\xbm)   \geq -\sqrt{\beta} \sigma_{t}(\xbm)\}\geq 1-\frac{\delta}{2}$.
\end{lemma}
\begin{proof}
 Under Assumption~\ref{assp:gp}, $f(\xbm)\sim \mathcal{N}(\mu_{t}(\xbm),\sigma_{t}^2(\xbm))$. By Lemma~\ref{lem:phi}, 
 \begin{equation} \label{eqn:fmu-a-pf-1}
 \centering
  \begin{aligned}
     \Pbb\left\{ f(\xbm)-\mu_t(\xbm) > \sqrt{\beta} \sigma_t(\xbm)\right\} \leq \frac{1}{2} e^{-\frac{\beta}{2}}.
  \end{aligned}
 \end{equation}
  Similarly, 
 \begin{equation} \label{eqn:fmu-a-pf-2}
 \centering
  \begin{aligned}
     \Pbb\left\{ f(\xbm)-\mu_t(\xbm) < -\sqrt{\beta} \sigma_t(\xbm)\right\} \leq \frac{1}{2} e^{-\frac{\beta}{2}}.
  \end{aligned}
 \end{equation}
  Thus, 
   \begin{equation} \label{eqn:fmu-a-pf-3}
 \centering
  \begin{aligned}
     \Pbb\left\{ |f(\xbm)-\mu_t(\xbm)| < \sqrt{\beta} \sigma_t(\xbm)\right\} \geq 1- e^{-\frac{\beta}{2}}.
  \end{aligned}
 \end{equation}
  Let $\beta$ be such that  
     $e^{-\frac{\beta}{2}} = \delta$
  and~\eqref{eqn:fmu-1} is proven. Similarly, from~\eqref{eqn:fmu-a-pf-1} and~\eqref{eqn:fmu-a-pf-2}, we obtain the one-sided inequalities.
\end{proof}
 Lemma~\ref{lem:fmu} is can be extended to hold over all $t\in\Nbb$ and $\xbm_t$ using union bound and an increasing $\beta_t$  in the following lemma (see~\cite{srinivas2009gaussian} for proof).   
\begin{lemma}\label{lem:fmu-t}
Given $\delta\in(0,1)$, let  $\beta_t = 2\log(\frac{\pi_t}{\delta})$, where $\pi_t=\frac{\pi^2t^2}{6}$. Under Assumption~\ref{assp:gp}, for  $\forall t\in\Nbb$,  the following holds 
\begin{equation} \label{eqn:fmu-t}
 \centering
  \begin{aligned}
     \Pbb\{|f(\xbm_{t+1}) - \mu_{t}(\xbm_{t+1})  | \leq \sqrt{\beta_{t+1}} \sigma_{t}(\xbm_{t+1}), \forall t\in\Nbb \} \geq 1-\delta. 
  \end{aligned}
\end{equation}
\end{lemma}
Next, we establish the relationship between $I_t$ and $EI_t(\xbm)$.
\begin{lemma}\label{lem:IEIbound}
Given $\delta\in(0,1)$, let $\beta = \max\{1.44,2\log(\frac{c_{\alpha}}{\delta})\}$, where constant $c_{\alpha} = \frac{1+2\pi}{2\pi}$. Under Assumption~\ref{assp:gp}, at given $ \xbm\in C$ and $t\in \Nbb$,   
\begin{equation} \label{eqn:IEI-bound-1}
 \centering
  \begin{aligned}
     \Pbb\left\{|I_t(\xbm) - EI_{t}(\xbm)  | \leq \sqrt{\beta} \sigma_{t}(\xbm)\right\}\geq 1-\delta. 
  \end{aligned}
\end{equation}
\end{lemma}
\begin{proof}
Given a scalar $w>1$, we consider the probabilities  
 \begin{equation} \label{eqn:IEI-bound-pf-3}
  \centering
  \begin{aligned}
      \Pbb\left\{ I_t(\xbm) > \sigma_{t}(\xbm) w + EI_{t}(\xbm)\right\} \quad \text{and} \quad  \Pbb\left\{I_t(\xbm)< -\sigma_{t}(\xbm) w +EI_{t}(\xbm)\right\}.
  \end{aligned}
\end{equation}
Consider the first probability in~\eqref{eqn:IEI-bound-pf-3}.
From Lemma~\ref{lem:EI}, $EI_{t}(\xbm)\geq 0$ for $\forall \xbm$ and $t$. Therefore, $\sigma_{t}(\xbm) w + EI_{t}(\xbm) > 0$. From Lemma~\ref{lem:EI}, Lemma~\ref{lem:Icdf},  and the monotonicity of $\Phi$, we have   
 \begin{equation} \label{eqn:IEI-bound-pf-4}
  \centering
  \begin{aligned}
       \Pbb\left\{ I_t(\xbm) > \sigma_{t} (\xbm) w + EI_{t}(\xbm)\right\} =&  1-\Phi\left(\frac{\sigma_{t}(\xbm) w+EI_{t}(\xbm)-y^+_{t}+ \mu_{t}(\xbm)}{\sigma_{t}(\xbm)} \right) \\
               \leq& 1-\Phi(w) \leq \frac{1}{2} e^{-\frac{w^2}{2}},
  \end{aligned}
\end{equation}
where the last inequality is from Lemma~\ref{lem:phi}.

For the second probability in~\eqref{eqn:IEI-bound-pf-3}, we further distinguish between two cases.
First, consider $-\sigma_{t}(\xbm) w +EI_{t}(\xbm)<0$. From Lemma~\ref{lem:Icdf},
 \begin{equation} \label{eqn:IEI-bound-pf-5}
  \centering
  \begin{aligned}
        \Pbb\left\{I_t(\xbm)< -\sigma_{t} (\xbm) w +EI_t(\xbm)\right\} = 0.
  \end{aligned}
\end{equation}
Second, let us consider the premise $ -\sigma_{t}(\xbm) w+ EI_{t}(\xbm)\geq 0$. By Lemma~\ref{lem:Icdf}, we have 
  \begin{equation} \label{eqn:IEI-bound-pf-6}
  \centering
  \begin{aligned}
       \Pbb\left\{I_t(\xbm)< -\sigma_{t}(\xbm) w +EI_{t}(\xbm)\right\} = \Phi\left( - w+\frac{EI_{t}(\xbm)-y^+_{t}+\mu_{t}(\xbm)}{\sigma_{t}(\xbm)} \right).
  \end{aligned}
\end{equation}
To proceed, we show that $y^+_{t}-\mu_{t}(\xbm) \geq 0$. Suppose on the contrary, $y^+_{t}-\mu_{t}(\xbm)<0$ and thus $z_{t}(\xbm)<0$. From Lemma~\ref{lem:EI}, 
  \begin{equation} \label{eqn:IEI-bound-pf-7}
  \centering
  \begin{aligned}
   \frac{EI_{t}(\xbm)}{ \sigma_{t}(\xbm)}  < \phi(z_{t}(\xbm)) \leq \phi(0) < 1 \leq w,
  \end{aligned}
  \end{equation}
  which contradicts the premise of this case.
Thus, we have $y^+_{t}-\mu_{t}(\xbm)\geq 0$ (and $z_{t}(\xbm)\geq 0$).
From the definition~\eqref{eqn:EI-1}, since $\Phi\in (0,1)$, 
  \begin{equation} \label{eqn:IEI-bound-pf-8}
  \centering
  \begin{aligned}
      \frac{EI_{t}(\xbm)-y^+_{t}+\mu_{t}(\xbm)}{\sigma_{t}(\xbm)}=& \left[z_{t}(\xbm)\left(\Phi(z_{t}(\xbm))-1\right)+\phi(z_{t}(\xbm))\right]           <\phi(z_{t}(\xbm)).
  \end{aligned}
  \end{equation}
In addition, by the premise of this case and Lemma~\ref{lem:EI},
  \begin{equation} \label{eqn:IEI-bound-pf-9}
  \centering
  \begin{aligned}
    w\leq  \frac{EI_{t}(\xbm)}{\sigma_{t}(\xbm)}
       \leq z_{t}(\xbm)+\phi(z_{t}(\xbm)).\\
  \end{aligned}
  \end{equation}
 Given that $w>1$ and $\phi(0)\geq \phi(z_{t}(\xbm))$, we have
  \begin{equation} \label{eqn:IEI-bound-pf-9.5}
  \centering
  \begin{aligned}
   z_{t}(\xbm)+\phi(0)>z_{t}(\xbm)+\phi(z_{t}(\xbm))>w, \ z_{t}(\xbm)>w-\phi(0)>0.
  \end{aligned}
  \end{equation}
As $z_{t}(\xbm)\geq 0$ increases, $\phi(z_{t}(\xbm))>0$ decreases. Thus, we have
\begin{equation} \label{eqn:IEI-bound-pf-10}
 \centering
  \begin{aligned}
    \frac{z_{t}(\xbm)}{\phi(z_{t}(\xbm))} > \frac{w-\phi(0)}{\phi(w-\phi(0))}, 
    \ \phi(z_{t}(\xbm))< \frac{\phi(w-\phi(0))}{w-\phi(0)} z_{t}(\xbm) .
  \end{aligned}
\end{equation}
  Denote $c_1(w) = \frac{w-\phi(0)}{w-\phi(0)+\phi(w-\phi(0))}$. 
  Applying~\eqref{eqn:IEI-bound-pf-10} to~\eqref{eqn:IEI-bound-pf-9}, we obtain
  \begin{equation} \label{eqn:IEI-bound-pf-11}
  \centering
  \begin{aligned}
    c_1(w) w < z_{t}(\xbm), \ \ \phi(z_{t}(\xbm)) < \phi( c_1(w) w) .\\
  \end{aligned}
  \end{equation}
  Applying~\eqref{eqn:IEI-bound-pf-11} and~\eqref{eqn:IEI-bound-pf-8} to~\eqref{eqn:IEI-bound-pf-6}, we obtain
  \begin{equation} \label{eqn:IEI-bound-pf-12}
  \centering
  \begin{aligned}
       \Pbb\left\{I_t(\xbm)< - w\sigma_t(\xbm) + EI_{t}(\xbm)\right\} <& \Phi\left( -w+ \phi(z_t(\xbm)) \right) \\
   <& \Phi\left( -w+ \phi(c_1(w) w) \right).
  \end{aligned}
\end{equation}
 Notice that $\phi(c_1(w) w)<\phi(c_1(w)) < \phi(c_1(w)) w$ due to $w>1$.
By the definition of $\Phi$ and mean value theorem, 
  \begin{equation} \label{eqn:IEI-bound-pf-13}
  \centering
  \begin{aligned}
        &\Phi\bigl(-w+  \phi(c_1(w) w) \bigr) = \Phi(-w)+ \int_{- w}^{-w+\phi(c_1(w) w)} \frac{1}{\sqrt{2\pi}} e^{-\frac{1}{2}x^2} dx
         \leq \Phi(-w)+\\&\frac{1}{\sqrt{2\pi}} e^{-\frac{1}{2} (w- \phi(c_1(w) w))^2} \phi(c_1(w) w)
         \leq \Phi(- w)+\frac{1}{2\pi} e^{-\frac{1}{2} ((1-\phi(c_1(w))) w)^2}  e^{-\frac{1}{2} (c_1(w) w)^2} \\
          &\leq \Phi(-w)+\frac{1}{2\pi} e^{-\frac{1}{2} c_2(w) w^2} 
          \leq \frac{1}{2} e^{-\frac{1}{2} w^2} + \frac{1}{2\pi} e^{-\frac{1}{2}c_2(w) w^2}, 
  \end{aligned}
\end{equation}
  where $c_2(w)=[1-\phi(c_1(w))]^2+[c_1(w)]^2$.
The last inequality in~\eqref{eqn:IEI-bound-pf-13} again uses Lemma~\ref{lem:phi}.
  Notice that $c_2(w)$ increases with $w$ and for $w\geq 1.2$, $c_2(w)>1$.
  Thus, $e^{-\frac{1}{2} w^2} > e^{-\frac{1}{2} c_2(w) w^2}$ for $w\geq 1.2$, which simplifies~\eqref{eqn:IEI-bound-pf-13} to 
  \begin{equation} \label{eqn:IEI-bound-pf-13.5}
  \centering
  \begin{aligned}
        \Phi\bigl( -w+ & \phi(c_1(w) w) \bigr) &< c_{\pi 1} e^{-\frac{1}{2} w^2}. 
  \end{aligned}
\end{equation}
where $c_{\pi 1}=\frac{1+\pi}{2\pi}$.
Therefore, by~\eqref{eqn:IEI-bound-pf-12} and~\eqref{eqn:IEI-bound-pf-13.5}, if $w\geq 1.2$, 
  \begin{equation} \label{eqn:IEI-bound-pf-14}
  \centering
  \begin{aligned}
       \Pbb\left\{I_t(\xbm)< -\sigma_{t}(\xbm) w +EI_{t}(\xbm)\right\} < c_{\pi 1}  e^{-\frac{1}{2} w^2}.
  \end{aligned}
\end{equation}
Combining~\eqref{eqn:IEI-bound-pf-14} with~\eqref{eqn:IEI-bound-pf-4} and~\eqref{eqn:IEI-bound-pf-5}, we have 
  \begin{equation} \label{eqn:IEI-bound-pf-15}
  \centering
  \begin{aligned}
       \Pbb\left\{ \left|I_t(\xbm) - EI_{t}(\xbm)\right| > w \sigma_{t}(\xbm) \right\} <c_{\alpha} e^{-\frac{1}{2} w^2},
  \end{aligned}
\end{equation}
  where $c_{\alpha}=\frac{1+2\pi}{2\pi}$ for $w\geq 1.2$.
  The probability in~\eqref{eqn:IEI-bound-pf-15} monotonically decreases with $w$. 
  Let 
    $\delta = c_{\alpha} e^{-\frac{1}{2} w^2}$. 
  Then, taking the logarithm of $\delta$ leads to 
    $\log( \frac{1+2\pi}{2\pi \delta}) = \frac{1}{2} w^2$. 
  Let $\beta=\max\{w^2,1.2^2\}$, and the proof is complete.
\end{proof}
Another relationship between $I_t$ and $EI_t$ under the GP prior assumption is given in the following lemma.
\begin{lemma}\label{lem:IEIbound-ratio}
Given $\delta\in(0,1)$, let $\beta = 2\log(\frac{1}{\delta})$. Under Assumption~\ref{assp:gp}, at given $\xbm\in C$ and $t\geq 1$, we have  
\begin{equation} \label{eqn:IEI-bound-ratio-1}
 \centering
  \begin{aligned}
    \Pbb\left\{\frac{\tau(-\sqrt{\beta})}{\tau(\sqrt{\beta})} I_t(\xbm) \leq  EI_{t}(\xbm)\right\}\geq 1-\delta.
  \end{aligned}
\end{equation}
\end{lemma}
\begin{proof}
   We consider two cases. First,  
   if $y^+_{t}-f(\xbm) \leq 0$, then $I_t(\xbm) = 0$. 
   Since $EI_t(\xbm)\geq 0$,~\eqref{eqn:IEI-bound-ratio-1} stands with probability $1$.

   Second, if  $y^+_{t}-f(\xbm) > 0$, then 
 \begin{equation} \label{eqn:IEI-bound-ratio-pf-1}
 \centering
  \begin{aligned}
     y^+_{t}-\mu_{t}(\xbm) &=      y^+_{t}-f(\xbm)+f(\xbm)-\mu_{t}(\xbm) >f(\xbm)-\mu_{t}(\xbm).\\
  \end{aligned}
 \end{equation}
   From the one-side inequality in Lemma~\ref{lem:fmu},~\eqref{eqn:IEI-bound-ratio-pf-1} implies 
 \begin{equation} \label{eqn:IEI-bound-ratio-pf-1.2}
 \centering
  \begin{aligned}
    \Pbb\{ y^+_{t}-\mu_{t}(\xbm) >-\sqrt{\beta}\sigma_{t}(\xbm)\} \geq 1-\frac{\delta}{2},
  \end{aligned}
\end{equation}
   Then, from Lemma~\ref{lem:tau} the monotonicity of $\tau(\cdot)$, we have 
 \begin{equation} \label{eqn:IEI-bound-ratio-pf-1.5}
 \centering
  \begin{aligned}
     \tau\left(z_t(\xbm)\right) >\tau(-\sqrt{\beta}),
  \end{aligned}
\end{equation}
  where $z_t(\xbm)=\frac{y^+_{t}-\mu_{t}(\xbm)}{\sigma_{t}(\xbm)}$,with probability greater than $1-\frac{\delta}{2}$. Since  $EI_t(\xbm) = \sigma_t(\xbm) \tau(z_t(\xbm))$, we can write 
  \begin{equation} \label{eqn:IEI-bound-ratio-pf-2}
 \centering
  \begin{aligned}
     EI_t(\xbm) = \sigma_t(\xbm)\tau\left(z_t(\xbm)\right) >\tau(-\sqrt{\beta})\sigma_{t}(\xbm),
  \end{aligned}
\end{equation}
  with probability greater than $ 1-\frac{\delta}{2}$.
  Next, we let $w=\sqrt{\beta}$ and consider the first probability in~\eqref{eqn:IEI-bound-pf-3}. 
  Since $\sqrt{\beta}\sigma_{t}(\xbm)+EI_t(\xbm)>0$, we can follow the proof of Lemma~\ref{lem:IEIbound} for this case and write similarly to~\eqref{eqn:IEI-bound-pf-4} that  
  \begin{equation} \label{eqn:IEI-bound-ratio-pf-3}
  \centering
  \begin{aligned}
       \Pbb\left\{ I_t(\xbm) > \sqrt{\beta}\sigma_{t} (\xbm) + EI_{t}(\xbm)\right\} \leq \frac{1}{2} e^{-\frac{\beta}{2}}=\frac{\delta}{2}.
  \end{aligned}
\end{equation}
  Therefore, 
  \begin{equation} \label{eqn:IEI-bound-ratio-pf-4}
  \centering
  \begin{aligned}
      \Pbb\left\{ I_t(\xbm)- EI_{t}(\xbm) \leq \sqrt{\beta}\sigma_{t}(\xbm) \right\}\geq 1-\frac{\delta}{2}. 
  \end{aligned}
\end{equation}
  Applying~\eqref{eqn:IEI-bound-ratio-pf-4} to~\eqref{eqn:IEI-bound-ratio-pf-2} by eliminating $\sigma_t(\xbm)$ and using union bound, we have  
 \begin{equation} \label{eqn:IEI-bound-ratio-pf-5}
 \centering
  \begin{aligned}
     EI_t(\xbm)  > \frac{\tau(-\sqrt{\beta})}{\sqrt{\beta}+\tau(-\sqrt{\beta})} I_t(\xbm) = \frac{\tau(-\sqrt{\beta})}{\tau(\sqrt{\beta})} I_t(\xbm),
  \end{aligned}
\end{equation}
  with probability greater than $1-\delta$.
\end{proof}

\section{Convergence analysis for GP-EI}\label{se:convergence}
  In this section, we provide the asymptotic convergence analysis for GP-EI under Assumption~\ref{assp:gp}. 
  Since $y^+_t$ is the best existing observation and an integral part in EI itself~\eqref{eqn:EI-1},  
  we use the error measure $r_t$, given $t$ samples, defined by 
 \begin{equation} \label{eqn:asymp-error}
 \centering
  \begin{aligned}
       r_t = y^+_{t}-f(\xbm^*),
  \end{aligned}
\end{equation}
 where $f(\xbm^*)=f^*$.
  The error~\eqref{eqn:asymp-error} is a natural extension of the noiseless error $f(\xbm^+_t)-f(\xbm^*)$ in~\cite{bull2011convergence} to the noisy case, which is also called simple regret. 
  Hence, the asymptotic convergence in the noisy case refers to convergnece of $y^+_t$ (with probability $1-\delta$ for $\forall \delta\in(0,1)$) to the vicinity of $f^*$ with values no greater than $f^*$, \textit{i.e.}, $r_t\leq 0$ as $t\to\infty$. 
  We note that $r_t$ is monotonically non-increasing, \textit{i.e.}, $r_{t+1}\leq r_{t}$.
  Our asymptotic convergence results of $r_t$ are general and can be used to establish convergence of other error measures, as elaborated in Remark~\ref{remark:EI-convergence-constant}.
  
  Assumptions~\ref{assp:gp},~\ref{assp:constraint} and~\ref{assp:gaussiannoise} are used in this section.
  In most of our analysis, we present error bounds in terms of $\sigma_t(\xbm_{t+1})$, which satisfies $\sigma_t(\xbm_{t+1})\to 0$ as $t\to\infty$, for many commonly-used kernels and GP-based algorithms~\cite{srinivas2009gaussian,bull2011convergence}. 
  In the noiseless case, the salient idea is that $\sigma_t(\xbm_{t+1})$ can be bounded by the distance between existing $t$ samples and any  $\xbm\in C$, leading to concrete convergence rate characterizations for specific kernels. 
  For an easy comparison, we adopt the kernel assumptions in~\cite{bull2011convergence} and provide the convergence rates in Theorem~\ref{prop:EI-convg-rate-nonoise}. 
  However, the techniques do not apply to the noisy case. 
  Instead, we prove convergence rates in the noisy case using information theory~\cite{srinivas2009gaussian} for two of the most widely used kernels: the SE kernel and Matérn kernel in Theorem~\ref{thm:EI-convg-rate}.

  For clarity of presentation, we briefly summarize the results in this section.
  Theorem~\ref{theorem:EI-convergence-1} establishes the asymptotic error bounds with and without noise under the GP prior assumption. We point out that the noiseless error bound is the same as that in~\cite{bull2011convergence}.
  Theorem~\ref{theorem:EI-convg-star} takes advantage of the exploration and exploitation properties of EI to establish a better error bound than Theorem~\ref{theorem:EI-convergence-1} with and without noise (under the GP prior assumption). 
  Theorem~\ref{prop:EI-convg-rate-nonoise} gives the convergence rate in the noiseless case under the GP prior assumption. 
  In Theorem~\ref{thm:EI-convg-rate},  the convergence rates in the noisy case using the improved error bounds from Theorem~\ref{theorem:EI-convg-star} are given for SE and Matérn kernels. 
  In Proposition~\ref{cor:EI-convg-compare}, we elucidate the improvement of Theorem~\ref{theorem:EI-convg-star} compared to 
   Theorem~\ref{theorem:EI-convergence-1}.

  \subsection{Asymptotic convergence in the noisy case}\label{se:asymp-noise}
  We state the boundedness of $f$ on $C$ as a lemma for easy reference. 
\begin{lemma}\label{lem:f-bound}
  There exists $M>0$, such that $|f(\xbm)|\leq M$ for all $\xbm\in C$.
\end{lemma}
 \begin{theorem}\label{theorem:EI-convergence-1}
   Given $\delta\in(0,1)$, let $\beta= 2 \log ( \frac{6}{\delta})$ and $c_t^{\sigma}=2\log(\frac{\pi^2t^2}{2\delta})$.
   The error bound of GP-EI satisfies
      \begin{equation} \label{eqn:EI-convg-1}
  \centering
  \begin{aligned}
            \Pbb\left\{   r_t \leq   c_{\tau}(\beta) \left[6\frac{M+\sqrt{c_t^\sigma} \sigma}{t-3} +(\sqrt{\beta}+\phi(0)) \sigma_{t_k}(\xbm_{t_k+1})\right]\right\}\geq 1-\delta,\\
   \end{aligned}
  \end{equation} 
 for given $t\geq \frac{3}{\log(2)}\log(\frac{3}{\delta})+3$,  where $c_{\tau}(\beta)=\frac{\tau(\sqrt{\beta})}{\tau(-\sqrt{\beta})}$ and $t_k\in[\frac{t}{3}-1, t]$.
   If there is no observation noise, choose  $\beta= 2 \log ( \frac{2}{\delta})$. Then,  
   \begin{equation} \label{eqn:EI-convg-2}
  \centering
  \begin{aligned}
              \Pbb\left\{ r_t \leq   c_{\tau}(\beta) \left[4\frac{M}{t-2} +(\sqrt{\beta}+\phi(0)) \sigma_{t_k}(\xbm_{t_k+1})\right]\right\}\geq 1-\delta,\\
   \end{aligned}
  \end{equation} 
 for $t_k\in[\frac{t}{2}-1, t]$.
\end{theorem}
\begin{proof}
   Since $\epsilon_t\sim\mathcal{N}(0,\sigma^2)$, by Lemma~\ref{lem:phi}, 
      $\Pbb\left\{  \epsilon_t  \leq \sqrt{w}\sigma\right\}= 1-\Phi(-\sqrt{w}) \geq 1-\frac{1}{2}e^{-\frac{1}{2}w}$, 
 for any given $t\in\Nbb$, where $w>0$ is a scalar. 
  Similar to Lemma~\ref{lem:fmu-t}, if $w_t=2\log(\frac{\pi^2t^2}{6\delta})$, we can use union bound and write 
   $\Pbb\left\{  |\epsilon_t|  \leq  \sqrt{w_t}\sigma\right\} \geq 1-\delta$,     
  for all $t\in\Nbb$.
  Let $c^{\sigma}_t=w_t=2\log(\frac{\pi^2t^2}{2\delta})$ such that 
\begin{equation} \label{eqn:EI-convg-pf-1}
  \centering
  \begin{aligned}
   \Pbb\left\{  |\epsilon_t|  \leq  \sqrt{c_t^{\sigma}}\sigma\right\}  \geq 1-\frac{\delta}{3},     
  \end{aligned}
  \end{equation}
  for all $t\in\Nbb$.
  From Lemma~\ref{lem:f-bound},~\eqref{eqn:EI-convg-pf-1}, and union bound, we have with probability greater than $1-\frac{\delta}{3}$ that
  \begin{equation} \label{eqn:EI-convg-pf-2}
  \centering
  \begin{aligned}
        \sum_{t=0}^{T-1}  y^+_{t}-y^+_{t+1} = y^+_0 - y_T^+ = f(\xbm)+\epsilon_0-f(\xbm^+_T)-\epsilon_T^+ \leq 2 M +2\sqrt{c^{\sigma}_t}\sigma. 
  \end{aligned}
  \end{equation}
  Given that $y^+_{t}-y^+_{t+1}\geq 0$, $y^+_{t}-y^+_{t+1} \geq \frac{2M +2\sqrt{c_t^{\sigma}}\sigma}{k}$ at most $k$ times for any $k\in \Nbb$ with probability greater than $1-\frac{\delta}{3}$. 

  Consider the observation noise with $\sigma>0$. 
  Let $A_k$ be the index set of $k$ samples, where $|A_k|=k$, $k\in\Nbb$.
  For a sequence of iterates $\xbm_i,i\in A_k$, the probability of $f(\xbm_i) < y_i = f(\xbm_i)+\epsilon_i$ for all $i\in A_k$ is $(\frac{1}{2})^k$, since $\epsilon_i$ is i.i.d. Gaussian. Consequently, the probability that there exists $i\in A_k$ such that $f(\xbm_i)\geq y_i$ is $1-(\frac{1}{2})^k$.
  
Let $k=[\frac{t}{3}]$ where $[x]$ is the largest integer smaller than $x$. Then, $3k\leq t\leq 3(k+1)$.  
  Now, consider $\xbm_i,y_i$ where $i\in [k,3k]$. From the discussion above, there exists at least $k+1$ different $i$ such that $y_i^+-y_{i+1}^+ <\frac{2M+2\sqrt{c_t^{\sigma}}\sigma}{k}$ with probability $1-\frac{\delta}{3}$.
  Thus, the probability that there exists $i\in [k,3k]$ among these $k+1$ iterates such that $f(\xbm_i)\geq y_i$ can be obtained via union bound to be greater than $1-\frac{\delta}{3}-(\frac{1}{2})^k$.   
For $t\geq \frac{3}{\log(2)}\log(\frac{3}{\delta})+3$, we know $(\frac{1}{2})^k\leq (\frac{1}{2})^{\frac{t}{3}-1}\leq \frac{\delta}{3}$.
  Thus, there exists $k \leq t_k \leq 3k$ such that 
  \begin{equation} \label{eqn:EI-convg-pf-2.5}
  \centering
  \begin{aligned}
    y_{t_k}^+-y_{t_k+1}^+<\frac{2M+2\sqrt{c_t^{\sigma}}\sigma}{k}, \ f(\xbm_{t_k+1})\geq y_{t_k+1}\geq y_{t_k+1}^+,
  \end{aligned}
  \end{equation}
with probability greater than $1-\frac{2\delta}{3}$ for $t\geq \frac{3}{\log(2)}\log(\frac{3}{\delta})+3$.
  
   The choice of $\beta$ satisfies both Lemma~\ref{lem:fmu} and~\ref{lem:IEIbound-ratio}  with $\frac{\delta}{6}$ each. 
    Using $\frac{\delta}{6}$ and $t_k$ in Lemma~\ref{lem:IEIbound-ratio}, one has with probability greater than $1-\frac{\delta}{6}$ that 
 \begin{equation} \label{eqn:EI-convg-pf-3}
  \centering
  \begin{aligned}
          r_t =& y^+_{t}-f(\xbm^*)\leq y^+_{t_k}-f(\xbm^*)=I_{t_k}(\xbm^*)\\ 
                  \leq& \frac{\tau(\sqrt{\beta})}{\tau(-\sqrt{\beta})} EI_{t_k}(\xbm^*) 
                  \leq \frac{\tau(\sqrt{\beta})}{\tau(-\sqrt{\beta})} EI_{t_k}(\xbm_{t_k+1}) \\
                  =& c_{\tau}(\beta)\left[(y^+_{t_k}-\mu_{t_k}(\xbm_{t_k+1}))\Phi(z_{t_k}(\xbm_{t_k+1}))+\sigma_{t_k}(\xbm_{t_k+1})\phi(z_{t_k}(\xbm_{t_k+1}))\right]. \\
  \end{aligned}
  \end{equation}
  Using $\frac{\delta}{6}$ and $t_k$ in Lemma~\ref{lem:fmu}, from~\eqref{eqn:EI-convg-pf-2.5}, we obtain 
   \begin{equation} \label{eqn:EI-convg-pf-4}
  \centering
  \begin{aligned}
                  y^+_{t_k}-\mu_{t_k}(\xbm_{t_k+1})
                  =& y^+_{t_k}-y^+_{t_k+1}+y_{t_k+1}^+-f(\xbm_{t_k+1})+f(\xbm_{t_k+1})-\mu_{t_k}(\xbm_{t_k+1}),\\
       \leq& \frac{2M+2\sqrt{c_t^{\sigma}} \sigma}{k}+ \sqrt{\beta}\sigma_{t_k}(\xbm_{t_k+1}),
  \end{aligned}
  \end{equation}
 for $t\geq \frac{3}{\log(2)}\log(\frac{3}{\delta})+3$, with probability greater than $1-\frac{5\delta}{6}$.
  Applying~\eqref{eqn:EI-convg-pf-4} to~\eqref{eqn:EI-convg-pf-3} and using $\Phi(\cdot)\in (0,1)$ and $\phi(\cdot)\leq \phi(0)$, we have
\begin{equation} \label{eqn:EI-convg-pf-5}
  \centering
  \begin{aligned}
         \Pbb\left\{ r_t  \leq c_{\tau}({\beta}) \left[\frac{2M+2\sqrt{c_t^{\sigma}} \sigma}{k} +(\sqrt{\beta}+\phi(0))\sigma_{t_k}(\xbm_{t_k+1})\right]\right\}\geq 1-\delta,\\
  \end{aligned}
  \end{equation}
 for $t\geq \frac{3}{\log(2)}\log(\frac{3}{\delta})+3$.

  Next, consider the noiseless case. 
  Choose $k=[\frac{t}{2}]$ so that $2k\leq t\leq 2(k+1)$.  There exists $k\leq t_k\leq 2k$ so that $y^+_{t_k}-y^+_{t_k+1} < \frac{2M}{k}$. 
   Using $\frac{\delta}{2}$ and $t_k$ in Lemma~\ref{lem:IEIbound-ratio}, we have with probability greater than $1-\frac{\delta}{2}$,~\eqref{eqn:EI-convg-pf-3} stands.
  Given no noise, $y^+_{t}\leq f(\xbm_t)$ for all $t\in\Nbb$. Using $\frac{\delta}{2}$ and $t_k$ in Lemma~\ref{lem:fmu}, we have 
  \begin{equation} \label{eqn:EI-convg-pf-7}
  \begin{aligned}
  \centering
   y^+_{t_k}-\mu_{t_k}(\xbm_{t_k+1})=&y^+_{t_k}-y^+_{t_k+1}+y^+_{t_k+1}-f(x_{t_k+1})+f(x_{t_k+1})-\mu_{t_k}(\xbm_{t_k+1})\\
   <& \frac{2M}{k} + f(\xbm_{t_k+1})-\mu_{t_k}(\xbm_{t_k+1})\leq \frac{2M}{k}+\sqrt{\beta}\sigma_{t_k}(\xbm_{t_k+1}),
   \end{aligned}
  \end{equation}
 with probability greater than $1-\frac{\delta}{2}$.
   From~\eqref{eqn:EI-convg-pf-7} and Lemma~\ref{lem:phi},~\eqref{eqn:EI-convg-pf-3} simplifies to 
  \begin{equation} \label{eqn:EI-convg-pf-8}
  \centering
  \begin{aligned}
   \Pbb\left\{  r_t    \leq c_{\tau}({\beta}) \left[\frac{2M }{k}+(\sqrt{\beta}+\phi(0))\sigma_{t_k}(\xbm_{t_k+1})\right]\right\}\geq 1-\delta.
  \end{aligned}
  \end{equation}

\end{proof}

\begin{remark}\label{remark:EI-convergence-constant}
The constant $c_{\tau}(\beta)$ in Theorem~\ref{theorem:EI-convergence-1} is dependent on $\delta$ and increases quickly as $\delta$ decreases.
The iteration number $t_k$ increases with $t$. 
Therefore, if $\sigma_t(\xbm_{t+1})\to 0$, then $\sigma_{t_k}(\xbm_{t_k+1})\to 0$ as well.

As mentioned above, the asymptotic convergence results of $r_t$ can be used to show convergence of other error measures.
For instance, one can use the function value $f$ instead of the observation to define  
       $\tilde{r}_t = f(\xbm_t)-f(\xbm^*)\geq 0.$
  The convergence result of $r_t$ implies the asymptotic convergence of $\tilde{r}_t\to 0$ if one adopts a stopping criterion $EI_t(\xbm_{t+1})\geq \kappa, \kappa>0$, similar to~\cite{nguyen17a}. 
  Another reasonable error measure for GP-EI is 
       $r_t' = I_t(\xbm^*) = \max\{y^+_{t}-f(\xbm^*),0\} \geq 0$,
  which includes the improvement function $I_t$, an important part of EI. 
  Again, the convergence result $r_t\leq 0$ as $t\to\infty$ ensures $r_t'\to 0$.
\end{remark}
\begin{remark}\label{remark:EI-convergence-bull}
Using the kernel assumptions (1)-(4) on $k(\xbm,\xbm')$ in~\cite{bull2011convergence}, Theorem~\ref{theorem:EI-convergence-1} leads to the $k^{-\frac{\min\{\nu,1\}}{d}}\log^{\beta}(k)$ rate for $\sigma_{t_k}(\xbm_{t_k+1})$ where $k\leq t_k\leq 3k$, and the same rate for $r_t$ as in~\cite{bull2011convergence}. We formalize the rates in the noiseless case in Theorem~\ref{prop:EI-convg-rate-nonoise}. However, these rates do not hold in the noisy case as $\sigma_t$ does not follow Lemma 7 in~\cite{bull2011convergence} anymore. 
The noisy case generates a larger upper bound with a larger $t$ for the same probability. 
We prove convergence rates for two kernels using information theory in the noisy case in Theorem~\ref{thm:EI-convg-rate}.
\end{remark}

\begin{remark}\label{remark:EI-convergence-difficulty}
From a technical perspective, two main challenges existed for the extension of asymptotic convergence of~\cite{bull2011convergence}.
The first one is the lack of a pointwise upper bound in the order of $\sigma_t(\xbm)$ on the prediction error $f(\xbm)-\mu_t(\xbm)$ when $f$ is in RKHS.
For instance, if $f$ is in RKHS with sub-Gaussian noises $\epsilon_i,i=1,\dots,t$,~\cite{chowdhury2017kernelized} proves $|f(\xbm)-\mu_t(\xbm)|\leq\sqrt{\beta_t}\sigma_t(\xbm)$ with high probability, where $\beta_t\to\infty$ as $t\to\infty$.
Applying such a $\beta_t$ to the asymptotic analysis leads to $c_{\tau}(\beta_t)$ that increases exponentially with $\sqrt{\beta_t}$ as $t\to\infty$. Hence, it is difficult to prove the boundedness of $r_t$. 
The second challenge is the loss of the inequalities $y^+_t\leq f(\xbm_t)$ and $f^+_{t+1}\leq f^+_t$.
We overcome these challenges by adopting the GP prior assumption, which allows for a pointwise prediction error inequality, \textit{i.e.}, Lemma~\ref{lem:fmu} at given $\xbm\in C$ and $t$. Thus, $\beta$ is based on the probability $\delta$ and does not increase with $t$. 
In addition, we use the monotonicity of best observations, \textit{i.e.}, $y^+_t-y^+_{t+1}\geq 0$, while the i.i.d. Gaussian noise gives a probability lower bound where $y^+_i\leq f(\xbm_i), i\in A_k, |A_k|=k$ for a given $k$.
\end{remark}
An improved error bound compared to Theorem~\ref{theorem:EI-convergence-1} is given next by analyzing and utilizing the exploration and exploitation properties of EI under both noiseless and noisy cases. Examples of constants $C_1$ and $C_2$ are provided in Remark~\ref{remark:c1c2}. 
\begin{theorem}\label{theorem:EI-convg-star}
   Given $\delta\in(0,1)$, choose $\beta= 2\log (\frac{9c_{\alpha}}{\delta})$, $w=\sqrt{2\log(\frac{9}{2\delta})}$, $c_t^{\sigma}=2\log(\frac{\pi^2t^2}{2\delta})$, and $c_{\alpha}=\frac{1+2\pi}{2\pi}$.
   Choose the constants $C_1 = \frac{1}{\Phi(-w)}$ and $C_2= \frac{\phi(0)}{\Phi(-w)}   + \sqrt{\beta}$.
   The GP-EI algorithm leads to the error bound
      \begin{equation} \label{eqn:EI-convg-star-1}
  \centering
  \begin{aligned}
     \Pbb\left\{ r_t  \leq C_1 (M+\sqrt{c_t^\sigma}\sigma)\frac{6}{t-3}+ (C_1\sqrt{\beta}+C_2) \sigma_{t_k}(\xbm_{t_k+1})\right\}\geq 1-\delta, 
   \end{aligned}
  \end{equation} 
  for $t\geq \frac{3\log(\frac{3}{\delta})}{\log(2)}+3$  and some $t_k\in [\frac{t}{3}-1,t]$. 
  If there is no noise,
   choose $\beta= 2\log (\frac{3c_{\alpha}}{\delta})$,
$C_1 = \frac{1}{\Phi(-\sqrt{\beta})}$ and $C_2= \frac{\phi(0)}{\Phi(-\sqrt{\beta})}   + \sqrt{\beta}$. Then, 
  \begin{equation} \label{eqn:EI-convg-star-noiseless}
  \centering
  \begin{aligned}
      \Pbb\left\{r_t  \leq C_1M\frac{4}{t-2}+ (C_1\sqrt{\beta}+C_2) \sigma_{t_k}(\xbm_{t_k+1})\right\}\geq 1-\delta, 
   \end{aligned}
  \end{equation} 
  where $t_k\in[\frac{t}{2}-1,t]$.
\end{theorem}
\begin{proof} 
   We first note that the choice of $\beta$ satisfies $\beta>1.44$ and, thus, is larger than the values required in Lemma~\ref{lem:IEIbound} and Lemma~\ref{lem:fmu} for $\frac{\delta}{9}$, \textit{i.e.},~\eqref{eqn:IEI-bound-1} and~\eqref{eqn:fmu-1} stand with probability greater than $1-\frac{\delta}{9}$ each. 
  Similar to the proof of Theorem~\ref{theorem:EI-convergence-1}, we know that with probability greater than $1-\frac{\delta}{3}$,~\eqref{eqn:EI-convg-pf-2} stands.
  Therefore, $y^+_{t}-y^+_{t+1} \geq \frac{2M +2\sqrt{c_t^{\sigma}}\sigma}{k}$ at most $k$ times for any $k\in \Nbb$.

  Given a sequence of iterates $\xbm_i$  of size $k\in \Nbb$ where $i\in A_k$,  $A_k$ being the index set of samples, the probability of $f(\xbm_i) < y_i = f(\xbm_i)+\epsilon_i$ for all $i\in A_k$ is $(\frac{1}{2})^k$.
  Let $k=[\frac{t}{3}]$ so that $3k\leq t\leq 3(k+1)$.
  Let $t\geq \frac{3}{\log(2)}\log(\frac{3}{\delta})+3$, which implies $(\frac{1}{2})^k\leq (\frac{1}{2})^{\frac{t}{3}-1}\leq \frac{\delta}{3}$.
  Thus, similar to~\eqref{eqn:EI-convg-pf-2.5}, there exists $k \leq t_k \leq 3k$ such that 
  \begin{equation} \label{eqn:EI-convg-star-pf-1}
  \centering
  \begin{aligned}
    \Pbb\left\{y_{t_k}^+-y_{t_k+1}^+<\frac{2M+2\sqrt{c_t^{\sigma}}\sigma}{k}, \ f(\xbm_{t_k+1})\geq y_{t_k+1}\geq y_{t_k+1}^+\right\} \geq 1-\frac{2\delta}{3}, 
   \end{aligned}
  \end{equation}
  whenever $t\geq \frac{3}{\log(2)}\log(\frac{3}{\delta})+3$. 

  From the monotonicity of $y_t^+$,~\eqref{eqn:acquisition-1}, and Lemma~\ref{lem:IEIbound}, one obtains 
  \begin{equation} \label{eqn:EI-convg-star-pf-2}
  \centering
  \begin{aligned}
          r_t =& y^+_{t}-f(\xbm^*)\leq y^+_{t_k}-f(\xbm^*)\leq I_{t_k}(\xbm^*)\\
               \leq& EI_{t_k}(\xbm^*) + \sqrt{\beta}\sigma_{t_k}(\xbm^*)
               \leq EI_{t_k}(\xbm_{t_k+1}) + \sqrt{\beta}\sigma_{t_k}(\xbm^*)\\
                =& (y^+_{t_k}-\mu_{t_k}(\xbm_{t_k+1}))\Phi(z_{t_k}(\xbm_{t_k+1}))+\sigma_{t_k}(\xbm_{t_k+1})\phi(z_{t_k}(\xbm_{t_k+1}))  + \sqrt{\beta}\sigma_{t_k}(\xbm^*)\\
            \leq& (y^+_{t_k}-\mu_{t_k}(\xbm_{t_k+1}))\Phi(z_{t_k}(\xbm_{t_k+1}))+\phi(0)\sigma_{t_k}(\xbm_{t_k+1}) + \sqrt{\beta}\sigma_{t_k}(\xbm^*),
   \end{aligned}
  \end{equation}
   with probability greater than $1-\frac{\delta}{9}$,
   where the last inequality uses Lemma~\ref{lem:phi}.
    From Lemma~\ref{lem:fmu}, we can write  
   \begin{equation} \label{eqn:EI-convg-star-pf-fmutk}
  \centering
  \begin{aligned}
        \Pbb\left\{|f(\xbm_{t_k+1})-\mu_{t_k}(\xbm_{t_k+1})|\leq \sqrt{\beta} \sigma_{t_k}(\xbm_{t_k+1})\right\}\geq 1-\frac{\delta}{9}.
    \end{aligned}
  \end{equation}
      From~\eqref{eqn:EI-convg-star-pf-1} and~\eqref{eqn:EI-convg-star-pf-fmutk}, we have with probability greater than $1-\frac{7\delta}{9}$ that 
   \begin{equation} \label{eqn:EI-convg-star-pf-5}
  \centering
  \begin{aligned}
          y^+_{t_k}-\mu_{t_k}(\xbm_{t_k+1}) =& y^+_{t_k}-y^+_{t_k+1}+y^+_{t_k+1}-f(\xbm_{t_k+1})+f(\xbm_{t_k+1})-\mu_{t_k}(\xbm_{t_k+1}) \\
          <&  \frac{2M+2\sqrt{c_t^{\sigma}}\sigma}{k} +\sqrt{\beta}\sigma_{t_k}(\xbm_{t_k+1}).
    \end{aligned}
  \end{equation}
  Applying~\eqref{eqn:EI-convg-star-pf-5} to~\eqref{eqn:EI-convg-star-pf-2} and using $\Phi(\cdot)\in(0,1)$, we have with probability greater than $1-\frac{8\delta}{9}$ that 
  \begin{equation} \label{eqn:EI-convg-star-pf-6}
  \centering
  \begin{aligned}
          r_t  \leq& \frac{2M+2\sqrt{c_t^{\sigma}}\sigma}{k}+(\phi(0)+\sqrt{\beta})\sigma_{t_k}(\xbm_{t_k+1}) + \sqrt{\beta}\sigma_{t_k}(\xbm^*).\\
  \end{aligned}
  \end{equation}
  From~\eqref{eqn:EI-convg-star-pf-6}, in order to obtain the upper bound for $r_t$, we need to consider the convergence behavior of $\sigma_{t_k}(\xbm^*)$ with respect to $\sigma_{t_k}(\xbm_{t_k+1})$.
   By its definition, $\xbm_{t_k+1}$ generates the largest $EI_{t_k}$. That is,
   \begin{equation} \label{eqn:EI-convg-star-pf-7}
  \centering
  \begin{aligned}
        EI_{t_k}(\xbm_{t_k+1}) \geq EI_{t_k}(\xbm^*).
    \end{aligned}
  \end{equation}
   We next show that~\eqref{eqn:EI-convg-star-pf-7} guarantees a small $\sigma_{t_k}(\xbm^*)$ in some scenarios.

   For simplicity of presentation, define $w=\sqrt{2\log(\frac{9}{2\delta})}$.
   From the proof of Lemma~\ref{lem:fmu} and Lemma~\ref{lem:phi}, we know that for $w>0$, 
   \begin{equation} \label{eqn:EI-convg-star-pf-8}
   \centering
   \begin{aligned}
       \Pbb\{ f(\xbm^*)-\mu_{t_k}(\xbm^*) > -\sigma_{t_k}(\xbm^*) w\}=1-\Phi(-w)\geq 1-\frac{1}{2}e^{-\frac{w^2}{2}}= 1-\frac{\delta}{9}.
    \end{aligned}
   \end{equation}
   Then, with probability greater than $1-\frac{\delta}{9}$, $f(\xbm^*)-\mu_{t_k}(\xbm^*) > -w\sigma_{t_k}(\xbm^*)$. By~\eqref{eqn:EI-convg-star-pf-fmutk} and~\eqref{eqn:EI-convg-star-pf-8}, with probability greater than $1-\frac{2\delta}{9}$, one has
   \begin{equation} \label{eqn:EI-convg-star-pf-9}
   \centering
   \begin{aligned}
      &y^+_{t_k}-\mu_{t_k}(\xbm^*)>y^+_{t_k}-f(\xbm^*)- w\sigma_{t_k}(\xbm^*)\\
          &= y^+_{t_k}-\mu_{t_k}(\xbm_{t_k+1})+\mu_{t_k}(\xbm_{t_k+1})-f(\xbm_{t_k+1})+f(\xbm_{t_k+1})-f(\xbm^*)- w\sigma_{t_k}(\xbm^*)\\
           &\geq y^+_{t_k}-\mu_{t_k}(\xbm_{t_k+1}) +f(\xbm_{t_k+1})-f(\xbm^*)- w\sigma_{t_k}(\xbm^*)-\sqrt{\beta}\sigma_{t_k}(\xbm_{t_k+1}).
    \end{aligned}
   \end{equation}
   Notice that~\eqref{eqn:EI-convg-star-pf-6} and~\eqref{eqn:EI-convg-star-pf-9} come from~\eqref{eqn:EI-convg-star-pf-1},~\eqref{eqn:EI-convg-star-pf-2},~\eqref{eqn:EI-convg-star-pf-fmutk} and~\eqref{eqn:EI-convg-star-pf-8}. Therefore, by union bound,~\eqref{eqn:EI-convg-star-pf-6} and~\eqref{eqn:EI-convg-star-pf-9} both hold with probability greater than $1-\delta$.
   We consider two scenarios regarding the term $f(\xbm_{t_k+1})-f(\xbm^*)$ in~\eqref{eqn:EI-convg-star-pf-9}.
   
   \textbf{Scenario A} Suppose first that $f(\xbm_{t_k+1})-f(\xbm^*)$ satisfies
   \begin{equation} \label{eqn:EI-convg-star-pf-9.5}
   \centering
   \begin{aligned}
      f(\xbm_{t_k+1})-f(\xbm^*) \leq C_1\max\{y^+_{t_k}-\mu_{t_k}(\xbm_{t_k+1}),0\}+C_2 \sigma_{t_k}(\xbm_{t_k+1}).
    \end{aligned}
   \end{equation}
    From the monotonicity of $y_t^+$,~\eqref{eqn:EI-convg-star-pf-1},~\eqref{eqn:EI-convg-star-pf-5} and~\eqref{eqn:EI-convg-star-pf-9.5}, we have with probability greater than $1-\frac{7\delta}{9}$, that  
   \begin{equation} \label{eqn:EI-convg-star-pf-10}
   \centering
   \begin{aligned}
         r_t=& y^+_{t}-f(\xbm^*) \leq y^+_{t_k+1}-f(\xbm^*) \leq f(\xbm_{t_k+1})-f(\xbm^*) \\
           \leq& C_1 \frac{2M+2\sqrt{c_t^{\sigma}}\sigma}{k}+(C_1\sqrt{\beta}+C_2)\sigma_{t_k}(\xbm_{t_k+1}),
    \end{aligned}
   \end{equation}
   which then proves~\eqref{eqn:EI-convg-star-1}. We note that $y_{t_k+1}^+\leq f(\xbm_{t_k+1})$ is used again in~\eqref{eqn:EI-convg-star-pf-10}.
   
   Before proceeding to the next scenario, we state useful properties of the parameters $C_1$ and $C_2$. By their definitions and Lemma~\ref{lem:phi},
   \begin{equation} \label{eqn:EI-convg-star-pf-constant}
   \centering
   \begin{aligned}
     C_1 = \frac{1}{\Phi(-w)}>2, \  C_2 -\sqrt{\beta} = \frac{\phi(0)}{\Phi(-w)}> \frac{\phi(0)}{\frac{1}{2}e^{-\frac{w^2}{2}}}=2\phi(0)e^{\frac{w^2}{2}}>w+1,   
    \end{aligned}
   \end{equation}
  where the last inequality uses $w\geq \sqrt{2\log(\frac{9}{2})}$.

  \textbf{Scenario B} Next, we analyze the scenario of
   \begin{equation} \label{eqn:EI-convg-star-pf-11}
   \centering
   \begin{aligned}
     f(\xbm_{t_k+1})-f(\xbm^*) >  C_1\max\{y^+_{t_k}-\mu_{t_k}(\xbm_{t_k+1}),0\}+C_2\sigma_{t_k}(\xbm_{t_k+1}). 
    \end{aligned}
   \end{equation}
   Since the exploitation part $ y^+_{t_k}-\mu_{t_k}(\xbm_{t_k+1})$ is bounded by $\sigma_{t_k}(\xbm_{t_k+1})$ in~\eqref{eqn:EI-convg-star-pf-5} and, thus, decreases for large $t$, we show that by choosing $C_1$ and $C_2$ large enough, EI focuses on exploration and~\eqref{eqn:EI-convg-star-pf-7} ensures that $\sigma_{t_k}(\xbm^*)$ decreases. Specifically, $\sigma_{t_k}(\xbm^*)\leq \sigma_{t_k}(\xbm_{t_k+1})$  with a high probability under scenario B~\eqref{eqn:EI-convg-star-pf-11}. An illustration of this idea is shown in Figure~\ref{fig:contourEI}. 
  We prove by contradiction. Suppose on the contrary that  
  \begin{equation} \label{eqn:EI-convg-star-pf-premise}
  \centering
  \begin{aligned}
     \sigma_{t_k}(\xbm_{t_k+1})     < \sigma_{t_k}(\xbm^*)\leq 1,
  \end{aligned}
  \end{equation}
   with probability $1-\delta_{\sigma}$.
 
   To proceed, we use the notation of exploitation and exploration for EI~\eqref{eqn:EI-ab}. Let $a_{t_k}=y^+_{t_k}-\mu_{t_k}(\xbm_{t_k+1})$ and $b_{t_k}= \sigma_{t_k}(\xbm_{t_k+1})$, \textit{i.e.}, $EI(a_{t_k},b_{t_k})=EI_{t_k}(\xbm_{t_k+1})$.
   Based on~\eqref{eqn:EI-convg-star-pf-9} and~\eqref{eqn:EI-convg-star-pf-premise}, let $b^*_{t_k}=\sigma_{t_k}(\xbm^*)$ to write  
   \begin{equation} \label{eqn:EI-convg-star-pf-12}
   \centering
   \begin{aligned}
          a^*_{t_k} &= y^+_{t_k}-\mu_{t_k}(\xbm_{t_k+1})+C_1\max\{y^+_{t_k}-\mu_{t_k}(\xbm_{t_k+1}),0\}+C_2\sigma_{t_k}(\xbm_{t_k+1})\\
  &-w \sigma_{t_k}(\xbm^*)-\sqrt{\beta}\sigma_{t_k}(\xbm_{t_k+1}) =a_{t_k}+C_1\max\{a_{t_k},0\}-w b^*_{t_k}+(C_2-\sqrt{\beta})b_{t_k}. 
    \end{aligned}
   \end{equation}
   Thus, $b^*_{t_k}>b_{t_k}$ with probability $1-\delta_{\sigma}$ by~\eqref{eqn:EI-convg-star-pf-premise}.
   Moreover, using~\eqref{eqn:EI-convg-star-pf-9},~\eqref{eqn:EI-convg-star-pf-11}, and~\eqref{eqn:EI-convg-star-pf-12}, the inequality
   \begin{equation} \label{eqn:EI-convg-star-pf-12.5}
   \centering
   \begin{aligned}
      y^+_{t_k}-\mu_{t_k}(\xbm^*)\geq a^*_{t_k},
    \end{aligned}
   \end{equation}
   holds with probability greater than $1-\frac{2\delta}{9}$.
   Since EI is monotonically increasing with respect to both $a$ and $b$ by Lemma~\ref{lem:EI-ms}, from~\eqref{eqn:EI-convg-star-pf-12.5}, we have  
   \begin{equation} \label{eqn:EI-convg-star-pf-13}
   \centering
   \begin{aligned}
          EI_{t_k}(\xbm^*) =EI(y^+_{t_k}-\mu_{t_k}(\xbm^*),\sigma_{t_k}(\xbm^*)) \geq EI(a^*_{t_k},b^*_{t_k}),
    \end{aligned}
   \end{equation}
  with probability $\geq 1-\frac{2\delta}{9}$.
   Denote $C_3=C_2-\sqrt{\beta}$, where by~\eqref{eqn:EI-convg-star-pf-constant} $C_3>w+1$. We consider two cases regarding $b_{t_k}^*$ in~\eqref{eqn:EI-convg-star-pf-12}. 
   
   \textbf{Case 1} If 
   \begin{equation} \label{eqn:EI-convg-star-pf-13.2}
   \centering
   \begin{aligned}
          b^*_{t_k} < \frac{C_1}{w}\max\{a_{t_k},0\}+\frac{C_3}{w}b_{t_k},
    \end{aligned}
   \end{equation}
   then $a^*_{t_k} > a_{t_k}$ by~\eqref{eqn:EI-convg-star-pf-12}. Further, from~\eqref{eqn:EI-convg-star-pf-premise}, with probability $1-\delta_{\sigma}$,
   \begin{equation} \label{eqn:EI-convg-star-pf-13.3}
   \centering
   \begin{aligned}
       EI(a^*_{t_k},b^*_{t_k}) > EI(a_{t_k},b_{t_k})=EI_{t_k}(\xbm_{t_k+1}).
    \end{aligned}
   \end{equation}
   
   \textbf{Case 2} Next we consider the case where 
   \begin{equation} \label{eqn:EI-convg-star-pf-13.5}
   \centering
   \begin{aligned}
          b^*_{t_k} \geq \frac{C_1}{w}\max\{a_{t_k},0\}+\frac{C_3}{w}b_{t_k} \geq \left(\frac{C_3}{w}\right) b_{t_k}.
    \end{aligned}
   \end{equation} 
   From $\frac{C_3}{w}>1$, it is clear that in \textbf{Case 2},~\eqref{eqn:EI-convg-star-pf-premise} always holds. 
   We further distinguish between two cases based on the sign of $a_{t_k}$.

   \textbf{Case 2.1} 
   Consider $a_{t_k}<0$, which implies $z_{t_k}=\frac{a_{t_k}}{b_{t_k}}<0$. Using~\eqref{eqn:EI-convg-star-pf-12} and definition of $\tau$~\eqref{def:tau}, 
    \begin{equation} \label{eqn:EI-convg-star-pf-14}
   \centering
   \begin{aligned}
            EI(a^*_{t_k},b^*_{t_k}) =& EI(a_{t_k}+C_3b_{t_k} -w b^*_{t_k},b^*_{t_k}) = b^*_{t_k} \tau\left(\frac{a_{t_k}+C_3b_{t_k}-w b^*_{t_k}}{b^*_{t_k}}\right) \\
             =& \frac{1}{\rho_{t_k}}  b_{t_k} \tau\left((z_{t_k}+C_3 )\rho_{t_k}-w \right).
    \end{aligned}
   \end{equation}
   where $\rho_{t_k}=\frac{b_{t_k}}{b^*_{t_k}}$. Consider two cases based on the value of $z_{t_k}$.

   \textbf{Case 2.1.1} 
   Suppose $z_{t_k}> -C_3$. Define function $\tilde{\tau}(\rho):\Rbb\to\Rbb$ at given parameters $z,w,C_3$ as 
     \begin{equation} \label{eqn:EI-convg-star-pf-bartau}
   \centering
   \begin{aligned}
    \bar{\tau}(\rho;z,w,C_3) = 
\frac{1}{\rho} \tau\left( (z+ C_3)\rho-w\right),
    \end{aligned}
   \end{equation}
   where under the conditions of \textbf{Case 2} and \textbf{Case 2.1.1}, $\rho\in (0,\frac{w}{C_3}), -C_3< z < 0$. 
   For brevity, we omit the parameters $z$, $w$, and $C_3$ in function notations at times (\textit{e.g.}, we use $\bar{\tau}(\rho)$ instead of $\bar{\tau}(\rho;z,w,C_3)$).
   We next find the minimum of~\eqref{eqn:EI-convg-star-pf-bartau} (at given $z$, $w$, and $C_3$).
   Taking the derivative with $\rho$ using Lemma~\ref{lem:tau}, we have 
    \begin{equation} \label{eqn:EI-convg-star-pf-14.1}
   \centering
   \begin{aligned}
    \frac{d \bar{\tau}}{d \rho} =& -\frac{1}{\rho^2}\tau\left( (z+ C_3)\rho-w\right)+\frac{1}{\rho}\Phi((z+ C_3)\rho-w)(z+C_3)\\
          =&-\frac{1}{\rho^2}[((z+ C_3)\rho-w)\Phi\left( (z+ C_3)\rho-w\right)+\phi((z+ C_3)\rho-w)]+\\
         &\frac{1}{\rho}\Phi((z+ C_3)\rho-w)(z+C_3)
         =-\frac{1}{\rho^2}[-w\Phi\left( (z+ C_3)\rho-w\right)+\phi((z+ C_3)\rho-w)].
    \end{aligned}
   \end{equation}
   Define the function $\theta(\rho;z,w,C_3) = -w\Phi\left( (z+ C_3)\rho-w\right)+\phi((z+ C_3)\rho-w)$.
   By~\eqref{eqn:EI-convg-star-pf-14.1}, the sign of $\frac{d \bar{\tau}}{d \rho}$ is determined by $\theta(\rho)$.
  The derivative of $\theta$ with $\rho$ is
     \begin{equation} \label{eqn:EI-convg-star-pf-14.2}
   \centering
   \begin{aligned}
    \frac{d \theta}{d \rho} 
         =&  -w\phi((z+ C_3)\rho-w)(z+C_3) - \phi((z+ C_3)\rho-w)((z+ C_3)\rho-w)(z+C_3)\\
          =& -\phi((z+ C_3)\rho-w)(z+C_3)^2\rho < 0, 
    \end{aligned}
   \end{equation}
  for $\forall \rho \in(0,\frac{w}{C_3})$. Therefore, $\theta$ is monotonically decreasing with $\rho$. 
  Now let $\rho\to 0$. We have $\theta(\rho)\to -w\Phi(-w)+\phi(-w)=\tau(-w)>0$.
  Further, let $\rho \to \frac{w}{C_3}$. Then, $\theta(\rho)\to -w \Phi(\frac{w}{C_3}z) + \phi(\frac{w}{C_3}z)$. 
  Depending on the sign of $-w\Phi(\frac{w}{C_3}z)+\phi(\frac{w}{C_3}z)$, $\bar{\tau}$ has two different types of behavior.  

   \textbf{Case 2.1.1.1} 
   First, $-w\Phi(\frac{w}{C_3}z)+\phi(\frac{w}{C_3}z)\geq 0$, which means $\theta(\rho)>0$ for $\forall \rho\in(0,\frac{w}{C_3})$. By~\eqref{eqn:EI-convg-star-pf-14.1}, $\bar{\tau}(\rho;z,w)$ is monotonically decreasing with $\rho$ and 
      \begin{equation} \label{eqn:EI-convg-star-pf-14.3}
   \centering
   \begin{aligned}
    \bar{\tau}(\rho;z,w,C_3)  > \bar{\tau}\left(\frac{w}{C_3};z,w,C_3\right) = \frac{C_3}{w} \tau\left( \frac{w}{C_3}z\right).
    \end{aligned}
   \end{equation}
   Using $\frac{w}{C_3}<1$, $z<0$, and the monotonicity of $\tau$ (Lemma~\ref{lem:tau}), we have 
      \begin{equation} \label{eqn:EI-convg-star-pf-14.4}
   \centering
   \begin{aligned}
    \bar{\tau}(\rho;z,w,C_3)  > \frac{C_3}{w} \tau\left( \frac{w}{C_3}z\right)>\tau\left( \frac{w}{C_3}z\right)>\tau(z).
    \end{aligned}
   \end{equation}
  Applying~\eqref{eqn:EI-convg-star-pf-14.4} to~\eqref{eqn:EI-convg-star-pf-14} with $z=z_{t_k}$,
      \begin{equation} \label{eqn:EI-convg-star-pf-14.5}
   \centering
   \begin{aligned}
    EI(a^*_{t_k},b^*_{t_k}) =  b_{t_k}\bar{\tau}(\rho_{t_k};z_{t_k},w,C_3)  > b_{t_k} \tau(z_{t_k})=EI(a_{t_k},b_{t_k}).
    \end{aligned}
   \end{equation}

   \textbf{Case 2.1.1.2} 
   In the second case, $-w\Phi(\frac{w}{C_3}z)+\phi(\frac{w}{C_3}z) < 0$ implies that there exists a unique $\bar{\rho}\in(0,\frac{w}{C_3})$ as a stationary point for $\theta(\rho)$, \textit{i.e.},
      \begin{equation} \label{eqn:EI-convg-star-pf-14.6}
   \centering
   \begin{aligned}
     \theta(\bar{\rho})= -w\Phi((z+C_3)\bar{\rho}-w) + \phi((z+C_3)\bar{\rho}-w) =0.
    \end{aligned}
   \end{equation}
  We claim $\bar{\rho}$ is a local minimum of $\bar{\tau}$ via the second derivative test.  
  From~\eqref{eqn:EI-convg-star-pf-14.1},~\eqref{eqn:EI-convg-star-pf-14.2}, and~\eqref{eqn:EI-convg-star-pf-14.6}, the second-order derivative of $\bar{\tau}$ with $\rho$ at $\bar{\rho}$ is
\begin{equation} \label{eqn:EI-convg-star-pf-14.7}
   \centering
   \begin{aligned}
    \frac{d^2 \bar{\tau}}{d \rho^2}|_{\bar{\rho}} =& \frac{2}{\bar{\rho}^3}\theta(\bar{\rho})-\frac{1}{\bar{\rho}^2}\frac{d\theta}{d\rho}(\bar{\rho}) = \frac{1}{\rho}\phi((z+ C_3)\bar{\rho}-w)(z+C_3)^2>0.
    \end{aligned}
   \end{equation}
   Therefore, $\bar{\rho}$ is a local minimum of $\bar{\tau}$.
  Further, $\frac{d\bar{\tau}}{d\rho}<0$ for $\rho<\bar{\rho}$ and  $\frac{d\bar{\tau}}{d\rho}>0$ for $\rho>\bar{\rho}$. Therefore, $\bar{\rho}$ is the global minimum on $(0,\frac{w}{C_3})$; and by~\eqref{eqn:EI-convg-star-pf-14.6}, we can write 
   \begin{equation} \label{eqn:ei-convg-star-pf-14.8}
   \centering
   \begin{aligned}
    \bar{\tau}(\bar{\rho};z,w,C_3) &= \frac{1}{\bar{\rho}} [((z+ C_3)\bar{\rho}-w)\Phi\left( (z+ C_3)\bar{\rho}-w\right) + \phi\left( (z+ C_3)\bar{\rho}-w\right)]\\
         &=\frac{1}{\bar{\rho}} [((z+ C_3)\bar{\rho}-w)\Phi\left( (z+ C_3)\bar{\rho}-w\right) + w\Phi\left( (z+ C_3)\bar{\rho}-w\right)] \\&
         = (z+ C_3) \Phi\left( (z+ C_3)\bar{\rho}-w\right).
    \end{aligned}
   \end{equation}
   Denote the corresponding $\bar{\rho}$ for $z_{t_k}$ as $\bar{\rho}_{t_k}$. By~\eqref{eqn:ei-convg-star-pf-14.8}, we have 
   \begin{equation} \label{eqn:EI-convg-star-pf-14.9}
   \centering
   \begin{aligned}
     \bar{\tau}(\rho_{t_k};z_{t_k},w) -\tau(z_{t_k}) >& (z_{t_k}+ C_3) \Phi\left( (z_{t_k}+ C_3)\bar{\rho}_{t_k}-w\right) - \tau(z_{t_k})\\
    \end{aligned}
   \end{equation}
   If $z_{t_k} \in (\frac{C_3\bar{\rho}_{t_k}-w}{1-\bar{\rho}_{t_k}} ,0)$, or equivalently $(z_{t_k}+ C_3)\bar{\rho}_{t_k}-w<z_{t_k}$,
   then by $\phi(z_{t_k})<\phi(0)$,~\eqref{eqn:EI-convg-star-pf-14.9} becomes
   \begin{equation} \label{eqn:EI-convg-star-pf-14.10}
   \centering
   \begin{aligned}
     \bar{\tau}(\rho_{t_k};z_{t_k},w,C_3) -\tau(z_{t_k}) 
       >&(z_{t_k}+ C_3) \Phi\left( (z_{t_k}+ C_3)\bar{\rho}_{t_k}-w\right) - z_{t_k}\Phi(z_{t_k})-\phi(0)\\
      \geq z_{t_k}[\Phi( (z_{t_k}+ C_3)\bar{\rho}_{t_k}&-w) -\Phi(z_{t_k})] + C_3 \Phi\left( (z_{t_k}+ C_3)\bar{\rho}_{t_k}-w\right) -\phi(0)\\
      \geq C_3 \Phi( (z_{t_k}+ C_3)\bar{\rho}_{t_k}-w)  &-\phi(0)>C_3\Phi(-w)-\phi(0)=C_3\Phi(-w)-\phi(0)= 0,
    \end{aligned}
   \end{equation}
   where the second inequality uses the monotonicity of $\Phi(\cdot)$ and $z_{t_k}<0$; the last equality is from~\eqref{eqn:EI-convg-star-pf-constant}.
   If $z_{t_k} \leq \frac{C_3\bar{\rho}_{t_k}-w}{1-\bar{\rho}_{t_k}}$, we have $(z_{t_k}+ C_3)\bar{\rho}_{t_k}-w\geq z_{t_k}$. 
   Suppose now $z_{t_k}>-C_3+1$. By~\eqref{eqn:EI-convg-star-pf-14.9} and Lemma~\ref{lem:tauvsPhi}, we can write  
    \begin{equation} \label{eqn:EI-convg-star-pf-14.11}
   \centering
   \begin{aligned}
     \bar{\tau}(\rho_{t_k};z_{t_k},w,C_3) -\tau(z_{t_k}) 
       \geq& (z_{t_k}+ C_3) \Phi\left( z_{t_k}\right) - \tau(z_{t_k})>\Phi(z_{t_k})-\tau(z_{t_k}) >0.\\
    \end{aligned}
   \end{equation}
   Finally, if $z_{t_k}\leq -C_3+1$, from~\eqref{eqn:EI-convg-star-pf-constant}, $z_{t_k}<-w$. 
   By definition~\eqref{eqn:EI-convg-star-pf-bartau} and $C_3>w$, we have  
    \begin{equation} \label{eqn:EI-convg-star-pf-14.11.5}
   \centering
   \begin{aligned}
     \bar{\tau}(\rho_{t_k};z_{t_k},w,C_3) -\tau(z_{t_k}) 
       \geq& \frac{C_3}{w} \tau\left( -w\right) - \tau(z_{t_k})> \frac{C_3}{w} \tau\left( -w\right) - \tau(-w)>0.\\
    \end{aligned}
   \end{equation}
    Combining all three cases~\eqref{eqn:EI-convg-star-pf-14.10},~\eqref{eqn:EI-convg-star-pf-14.11} and~\eqref{eqn:EI-convg-star-pf-14.11.5},~\eqref{eqn:EI-convg-star-pf-14.9} leads to
   \begin{equation} \label{eqn:EI-convg-star-pf-14.12}
   \centering
   \begin{aligned}
     \bar{\tau}(\rho_{t_k};z_{t_k},w,C_3) -\tau(z_{t_k})>0. 
    \end{aligned}
   \end{equation}
   From~\eqref{eqn:EI-convg-star-pf-14},~\eqref{eqn:EI-convg-star-pf-bartau} and~\eqref{eqn:EI-convg-star-pf-14.12}, we again have~\eqref{eqn:EI-convg-star-pf-14.5}.
  Therefore,~\eqref{eqn:EI-convg-star-pf-14.5} holds for \textbf{Case 2.1.1}. 

   \textbf{Case 2.1.2} 
   Next, consider $z_{t_k}\leq -C_3 $. By~\eqref{eqn:EI-convg-star-pf-13.5}, 
  $\frac{b_{t_k}}{b^*_{t_k}}\leq \frac{w}{C_3}<1$.
    Using the monotonicity of $\tau(\cdot)$,~\eqref{eqn:EI-convg-star-pf-14} implies  
    \begin{equation} \label{eqn:EI-convg-star-pf-16}
   \centering
   \begin{aligned}
            EI(a^*_{t_k},b^*_{t_k}) =\frac{1}{\rho_{t_k}}b_{t_k} \tau\left(\rho_{t_k} (z_{t_k}+C_3)-w \right)
             >& \frac{C_3}{w} b_{t_k} \tau\left(\frac{w}{C_3} (z_{t_k}+C_3)-w \right) \\
             > & b_{t_k} \tau\left(z_{t_k} \right)= EI(a_{t_k},b_{t_k}).
    \end{aligned}
   \end{equation}
   That is, we have~\eqref{eqn:EI-convg-star-pf-14.5} for \textbf{Case 2.1.2}.
   Combining \textbf{Case 2.1.1} and \textbf{Case 2.1.2}, we have $EI(a^*_{t_k},b^*_{t_k})>EI(a_{t_k},b_{t_k})$ if $a_{t_k}<0$
 (\textbf{Case 2.1}). Figure~\ref{fig:bartau} demonstrates an example contour of $\bar{\tau}(\rho;z,w)-\tau(z)$  for some fixed $w$ and $C_3$.
    
   \textbf{Case 2.2}
   Next, we consider $a_{t_k}\geq 0$.  
   By definition~\eqref{eqn:EI-convg-star-pf-12}, we can write 
     \begin{equation} \label{eqn:EI-convg-star-pf-17}
   \centering
   \begin{aligned}
            EI(a^*_{t_k},b^*_{t_k}) =& EI((C_1+1)a_{t_k}+C_3 b_{t_k} -w b^*_{t_k},b^*_{t_k}) 
              \geq b^*_{t_k} 
\tau\left((C_1 z_{t_k}+C_3) \rho_{t_k} -w \right)\\
             =& b_{t_k} \frac{1}{\rho_{t_k}} \tau\left( C_1z_{t_k}\rho_{t_k}+ C_3\rho_{t_k}-w\right),
    \end{aligned}
   \end{equation}
   where $\rho_{t_k}=\frac{b_{t_k}}{b^*_{t_k}}$. The inequality in~\eqref{eqn:EI-convg-star-pf-17} uses $(C_1+1)a_{t_k}\geq C_1 a_{t_k}$ and the monotonicity of $\tau(\cdot)$.
    By~\eqref{eqn:EI-convg-star-pf-13.5} and $a_{t_k}\geq 0$,
     \begin{equation} \label{eqn:EI-convg-star-pf-17.2}
   \centering
   \begin{aligned}
      \rho_{t_k}\leq \left( \frac{w}{C_3}-\frac{C_1}{C_3} \frac{a_{t_k}}{b_{t_k}^*}\right) \leq \left( \frac{w}{C_3}-\frac{C_1}{C_3} a_{t_k}\right)\leq  \frac{w}{C_3}.
    \end{aligned}
   \end{equation}
  Thus,
       $C_3\rho_{t_k}-w \leq 0$.
  Define function $\tilde{\tau}(\rho,z):\Rbb^2\to\Rbb$ with parameter $w,C_1,C_3$ as 
     \begin{equation} \label{eqn:EI-convg-star-pf-tildetau}
   \centering
   \begin{aligned}
    \tilde{\tau}(\rho,z;w,C_1,C_3) = 
\frac{1}{\rho} \tau\left( C_1 z\rho+ C_3\rho-w\right),
    \end{aligned}
   \end{equation}
   for $\rho\in (0,\frac{w}{C_3}), z\geq 0$.
   From~\eqref{eqn:EI-convg-star-pf-17}, $EI(a^*_{t_k},b^*_{t_k})\geq b_{t_k}\tilde{\tau}(\rho_{t_k},z_{t_k})$.
    We now compare $ \tilde{\tau}(\rho_{t_k},z_{t_k};w,C_1,C_3)$ and $\tau(z_{t_k})$. For $\forall \rho_{t_k}\in (0,\frac{w}{C_3})$,~\eqref{eqn:EI-convg-star-pf-constant} and Lemma~\ref{lem:tau} lead to the derivative 
     \begin{equation} \label{eqn:EI-convg-star-pf-18}
   \centering
   \begin{aligned}
         \frac{\partial \tilde{\tau}}{\partial z} = C_1 \Phi\left(C_1z \rho +C_3\rho-w\right)>C_1\Phi(-w)> \Phi(z)=\frac{d \tau}{d z}>0.
    \end{aligned}
   \end{equation}
   Therefore, both $\tilde{\tau}(\cdot,\cdot)$ and $\tau(\cdot)$ are monotonically increasing with $z\geq 0$ and $\tilde{\tau}$ always has a larger positive derivative. For any $\rho_{t_k}$, this leads to
     \begin{equation} \label{eqn:EI-convg-star-pf-19}
   \centering
   \begin{aligned}
      \tilde{\tau}(\rho_{t_k},z_{t_k};w,C_1,C_3) - \tau(z_{t_k}) \geq \tilde{\tau}(\rho_{t_k},0;w,C_1,C_3)-\tau(0)=\frac{1}{\rho_{t_k}}\tau\left(C_3\rho_{t_k}-w \right)-\tau(0).
    \end{aligned}
    \end{equation}
  Notice that the right-hand side of~\eqref{eqn:EI-convg-star-pf-19} can be written via~\eqref{eqn:EI-convg-star-pf-bartau} as
     \begin{equation} \label{eqn:EI-convg-star-pf-19.1}
   \centering
   \begin{aligned}
      \bar{\tau}(\rho_{t_k};0,w,C_3)-\tau(0).
    \end{aligned}
    \end{equation}
   Therefore, we can follow the analysis in \textbf{Case 2.1.1} with $z=0$. Notice that $-w\Phi(0)+\phi(0)<0$ for $\forall w>1$. Thus,~\eqref{eqn:EI-convg-star-pf-19.1} follows \textbf{Case 2.1.1.2} where a minimum exists at $\bar{\rho}\in (0,\frac{w}{C_3})$ which satisfies $w\Phi(C_3\bar{\rho}-w)= \phi(C_3\bar{\rho}-w)$. Notice that $\bar{\rho}$ does not depend on $t_k$ anymore. Hence, we have  
     \begin{equation} \label{eqn:EI-convg-star-pf-19.2}
   \centering
   \begin{aligned}
      \bar{\tau}(\rho_{t_k};0,w,C_3)-\tau(0) >& \frac{1}{\bar{\rho}} \tau(C_3\bar{\rho}-w) - \phi(0) 
         =\frac{1}{\bar{\rho}} ((C_3\bar{\rho}-w) \Phi(C_3\bar{\rho}-w)+\\
         \phi(C_3\bar{\rho}-w))&-\phi(0)
         = C_3 \Phi(C_3\bar{\rho}-w)-\phi(0) > C_3\Phi(-w)-\phi(0).
    \end{aligned}
    \end{equation}
    Then,~\eqref{eqn:EI-convg-star-pf-constant} implies that 
     \begin{equation} \label{eqn:EI-convg-star-pf-19.3}
   \centering
   \begin{aligned}
      \bar{\tau}(\rho_{t_k};0,w,C_3)-\tau(0) > C_3\Phi(-w)-\phi(0) = 0.
    \end{aligned}
    \end{equation}
   By~\eqref{eqn:EI-convg-star-pf-17},~\eqref{eqn:EI-convg-star-pf-tildetau},~\eqref{eqn:EI-convg-star-pf-19},~\eqref{eqn:EI-convg-star-pf-19.1}, and~\eqref{eqn:EI-convg-star-pf-19.3}, we have 
     \begin{equation} \label{eqn:EI-convg-star-pf-20}
   \centering
   \begin{aligned}
            EI(a^*_{t_k},b^*_{t_k})- EI(a_{t_k},b_{t_k})\geq b_{t_k} \tilde{\tau}(\rho_{t_k},0;w,C_1,C_3)-b_{t_k}\tau(0)> 0.
    \end{aligned}
   \end{equation}
    Illustrative examples of $\tilde{\tau}$ is given in Figure~\ref{fig:tildetau}.

    To summarize \textbf{Scenario B Case 2}, if~\eqref{eqn:EI-convg-star-pf-13.5}, 
   then $EI(a_{t_k}^*,b_{t_k}^*)>EI(a_{t_k},b_{t_k})$.
   Combined with~\eqref{eqn:EI-convg-star-pf-13.3}, under \textbf{Scenario B}, $EI(a_{t_k}^*,b_{t_k}^*)>EI(a_{t_k},b_{t_k})$ holds with probability greater than $1-\delta_{\sigma}$.
   Combined with~\eqref{eqn:EI-convg-star-pf-13}, we have with probability greater than $1-\frac{2\delta}{9}-\delta_{\sigma}$ that $EI_{t_k}(\xbm^*)>EI_{t_k}(\xbm_{t_k+1})$. However, this is a contradiction of~\eqref{eqn:EI-convg-star-pf-7}, which is a sure event.
   Therefore, we have   
 \begin{equation} \label{eqn:EI-convg-star-pf-20.5}
   \centering
   \begin{aligned}
   1-\frac{2\delta}{9}-\delta_{\sigma}\leq 0.
    \end{aligned}
   \end{equation}
   That is, $\delta_{\sigma}\geq 1-\frac{2\delta}{9}$. Then,~\eqref{eqn:EI-convg-star-pf-premise} implies that 
 \begin{equation} \label{eqn:EI-convg-star-pf-21}
   \centering
   \begin{aligned}
           \Pbb\{\sigma_{t_k}(\xbm_{t_k+1})\geq \sigma_{t_k}(\xbm^*)\}\geq 1-\frac{2\delta}{9}.
    \end{aligned}
   \end{equation}
   We note again that~\eqref{eqn:EI-convg-star-pf-21} and~\eqref{eqn:EI-convg-star-pf-6} hold simultaneously with probability greater than $1-\delta$, as $\delta$ is derived from~\eqref{eqn:EI-convg-star-pf-1}, Lemma~\ref{lem:IEIbound}, and Lemma~\ref{lem:fmu}.
   Under \textbf{Scenario B}, by~\eqref{eqn:EI-convg-star-pf-21} and~\eqref{eqn:EI-convg-star-pf-6}, we obtain with probability greater than $1-\delta$ that 
     \begin{equation} \label{eqn:EI-convg-star-pf-22}
   \centering
   \begin{aligned}
       r_t \leq \frac{2M+2\sqrt{c_t^{\sigma}}\sigma}{k} + (\phi(0)+2\sqrt{\beta})\sigma_{t_k}(\xbm_{t_k+1}).
     \end{aligned}
   \end{equation}
   From~\eqref{eqn:EI-convg-star-pf-constant}, we know  $C_1> 2$ and $C_2>1$. 
    Combining \textbf{Scenario A}~\eqref{eqn:EI-convg-star-pf-10} and \textbf{Scenario B}~\eqref{eqn:EI-convg-star-pf-22}, one can obtain~\eqref{eqn:EI-convg-star-1}. 
  
   If there is no noise, \textit{i.e.}, $\sigma=0$,  we have $f(\xbm_t) =y_t \geq y_t^+$.
   Therefore, we can choose $k=[\frac{t}{2}]$ and there exists $k \leq t_k \leq 2k$ so that $y_{t_k}^+-y_{t_k+1}^+<\frac{2M}{k}$ with probability $1$. 
  Following the same proof procedure above in the noisy case, we obtain~\eqref{eqn:EI-convg-star-noiseless}.
\end{proof}
\subsection{Convergence rates and comparative analysis}\label{se:rate}
The rate of decrease for the noiseless error bound of Theorem~\ref{theorem:EI-convg-star} is given below in the setup of~\cite{bull2011convergence}.
Recall that $\nu$ is the Matérn kernel parameter.
Define 	$\eta=\begin{cases} \alpha, \ \nu\leq 1\\
			0, \ \nu>1,
		\end{cases}$
where $\alpha=\frac{1}{2}$ if $\nu\in\Nbb$, and $\alpha=0$ otherwise.
The readers are referred to~\cite{bull2011convergence} for the underlying assumptions on the kernels and the choice of parameters $\nu>0$ and $\alpha\geq 0$. 
\begin{theorem}\label{prop:EI-convg-rate-nonoise}
   Given $\delta\in(0,1)$, 
   let $\beta=2\log (\frac{3c_{\alpha}}{\delta})$, where $c_{\alpha}=\frac{1+2\pi}{\pi}$, $C_1 = \frac{1}{\Phi(-\sqrt{\beta})}$ and $C_2= \frac{\phi(0)}{\Phi(-\sqrt{\beta})}   + \sqrt{\beta}$.
   Suppose the kernels $k(\cdot,\cdot)$ satisfy assumptions 1-4 in~\cite{bull2011convergence}. Then, there exists constant $C'$ such that  
  \begin{equation} \label{eqn:EI-convg-rate-nonoise-1}
  \centering
  \begin{aligned}
     \Pbb\left\{ r_t \leq   C_1 M\frac{6}{t-3} + (C_1\sqrt{\beta}+C_2)C'  \left(\frac{3}{t-3}\right)^{\frac{ \min\{\nu,1\}}{d}}\log^{\eta}\left(\frac{t}{3}\right)\right\} \geq 1-\delta. 
   \end{aligned}
  \end{equation} 
  Hence, the convergence rate of GP-EI is $\mathcal{O}(t^{-\frac{ \min\{\nu,1\}}{d}}\log^{\eta}(t))$ in the noiseless case. 
\end{theorem}
\begin{proof}
     By Lemma 7 in~\cite{bull2011convergence}, there exists a constant $C'>0$ such that for $\forall k\in\Nbb$ and $\xbm_i,i=1,\dots,t$, $\sigma_{i}(\xbm_{i+1})\geq C' k^{-\frac{ \min\{\nu,1\}}{d}}\log^{\eta}(k)$ holds at most $k$ times, where $\eta = \alpha$ if $\nu\leq 1$ and $\eta=0$ if $\nu>1$. Both $\nu>0$ and $\alpha\geq 0$ are parameters defining the properties of the kernels in assumptions 1-4 in~\cite{bull2011convergence}.

    Following the proof of Theorem~\ref{theorem:EI-convg-star}, we let $k=\left[\frac{t}{3}\right]$. Then, there exists 
    $k\leq t_k \leq 3k$ so that $y_{t_k}^+-y_{t_k+1}<\frac{2M}{k}$ and $\sigma_{t_k}(\xbm_{t_k+1})< C' k^{-\frac{ \min\{\nu,1\}}{d}}\log^{\eta}(k)$. Applying Theorem~\ref{theorem:EI-convg-star}, with probability $1-\delta$, 
   \begin{equation} \label{eqn:EI-convg-rate-nonoise-pf-1}
  \centering
  \begin{aligned}
      r_t \leq& 2C_1 M\frac{1}{k} + (C_1\sqrt{\beta}+C_2)\sigma_{t_k}(\xbm_{t_k+1})\\
          \leq&  2C_1 M\frac{1}{k} + (C_1\sqrt{\beta}+C_2)C' k^{-\frac{ \min\{\nu,1\}}{d}}\log^{\eta}(k).
   \end{aligned}
  \end{equation} 
\end{proof}
\begin{remark}\label{remark:rate-nonoise-original}
   Theorem~\ref{prop:EI-convg-rate-nonoise} illustrates that the convergence rate in the noiseless case under the GP prior assumption is similar to that under the RKHS assumption in~\cite{bull2011convergence}.
   The analysis can be applied to Theorem~\ref{theorem:EI-convergence-1} in the noiseless case as well, with a similar convergence rate.
   We note that both the SE and Matérn kernels satisfy assumptions (1)-(4)~\cite{bull2011convergence}, with SE kernel obtained as $\nu\to\infty$.
\end{remark}
 The rate of decrease for the error bound with noise is given in the following theorem.
\begin{theorem}\label{thm:EI-convg-rate}
   Given $\delta\in(0,1)$, let $\beta= 2\log (\frac{9c_{\alpha}}{\delta})$, $w=\sqrt{2\log(\frac{9}{2\delta})}$, and $c_t^{\sigma}=2\log(\frac{\pi^2t^2}{2\delta})$, where $c_{\alpha}=\frac{1+2\pi}{2\pi}$.
   Let  $C_1 = \frac{1}{\Phi(-w)}$ and $C_2= \frac{\phi(0)}{\Phi(-w)}   + \sqrt{\beta}$.
   Then, for $t\geq\frac{4\log(\frac{3}{\delta})}{\log(2)}+4$, the error bound of GP-EI reduces at the rate 
     \begin{equation} \label{eqn:EI-convg-rate-1}
  \centering
  \begin{aligned}
    \mathcal{O} \left(t^{-\frac{1}{2}}\log(t)^{\frac{d+1}{2}}\right) \ \text{and }\ \mathcal{O}(t^{\frac{-\nu}{2\nu+d}}\log^{\frac{\nu}{2\nu+d}}(t)),
   \end{aligned}
  \end{equation} 
   with probability greater than $1-\delta$, for SE and Matérn kernels, respectively.
\end{theorem}
\begin{proof}
  From Lemma~\ref{lem:variancebound}, we know
        $\sum_{i=0}^{t-1}  \sigma_{i}^2(\xbm_{i+1}) \leq C_{\gamma} \gamma_t$,
  where $C_{\gamma} = \frac{2}{\log(1+\sigma^{-2})}$.
  Therefore, $\sigma_{i}^2(\xbm_{i+1})  \geq \frac{C_{\gamma} \gamma_t}{k}$ at most $k$ times for any $k\in \Nbb$.
  Choose $k=[\frac{t}{4}]$ so that $4k\leq t\leq 4(k+1)$.
  Then, for $t\geq\frac{4\log(\frac{3}{\delta})}{\log(2)}+4$, $(\frac{1}{2})^k\leq (\frac{1}{2})^{\frac{t}{4}-1}\leq\frac{\delta}{3}$.
  Following the proof of Theorem~\ref{theorem:EI-convg-star}, there exists $k\leq t_k\leq 4k$ where $y^+_{t_k}-y^+_{t_k+1} < \frac{2M +2\sqrt{c_t^{\sigma}}\sigma}{k}$, $f(\xbm_{t_k+1})\geq y_{t_k+1}$, and $\sigma_{t_k}^2(\xbm_{t_k+1})<\frac{C_{\gamma}\gamma_t}{k}$, with probability greater than $1-\frac{2\delta}{3}$ for $t\geq\frac{4\log(\frac{3}{\delta})}{\log(2)}+4$. 
  Using Theorem~\ref{theorem:EI-convg-star}, we can obtain 
     \begin{equation} \label{eqn:EI-convg-rate-pf-2}
  \centering
  \begin{aligned}
      r_t  \leq& C_1 (2M+2\sqrt{c_t^\sigma}\sigma)\frac{4}{t-4}+ (C_1\sqrt{\beta}+C_2) \sigma_{t_k}(\xbm_{t_k+1})\\
        \leq& C_1 (2M+2\sqrt{c_t^\sigma}\sigma)\frac{4}{t-4}+ 2 (C_1\sqrt{\beta}+C_2) \sqrt{\frac{C_{\gamma}\gamma_t}{t-4}},\\
   \end{aligned}
  \end{equation} 
  with probability greater than $1-\delta$.
   If $k(\cdot,\cdot)$  is the SE kernel, from Lemma~\ref{lem:gammarate},  the upper bound for $r_t$ is $\mathcal{O}(t^{-\frac{1}{2}}(\log t)^{\frac{d+1}{2}})$.
   Similarly, by Lemma~\ref{lem:gammarate} and~\eqref{eqn:EI-convg-rate-pf-2}, the rate of reduction of $r_t$ for Matérn kernel can be obtained as well.
\end{proof}
Next, we prove that Theorem~\ref{theorem:EI-convg-star} provides tighter bounds than Theorem~\ref{theorem:EI-convergence-1}.
\begin{proposition}\label{cor:EI-convg-compare}
   Using the same $\delta \in (0,1)$ in Theorem~\ref{theorem:EI-convergence-1} and~\ref{theorem:EI-convg-star}, 
  the error bound in Theorem~\ref{theorem:EI-convg-star} is smaller than that of Theorem~\ref{theorem:EI-convergence-1}.
\end{proposition}
\begin{proof}
   We prove the noisy case, as the noiseless case is similar.
   To distinguish the parameters, we denote the $\beta$ in Theorem~\ref{theorem:EI-convg-star} as $\beta_{4.6}$
   and $\beta$ in Theorem~\ref{theorem:EI-convergence-1} as $\beta_{4.2}$.
   Hence, given the same probability $\delta$, $\beta_{4.6}=2\log(\frac{9c_{\alpha}}{\delta})$ and the $\beta_{4.2} = 2\log(\frac{6}{\delta})$.
    To better compare the two theorems, we can rewrite both ~\eqref{eqn:EI-convg-1} and~\eqref{eqn:EI-convg-star-1} as 
    \begin{equation} \label{eqn:EI-convg-compare-bound}
  \centering
  \begin{aligned}
               r_t \leq  3C_{4} \frac{2M+2\sqrt{c_t^\sigma} \sigma}{t-3} +  C_5 \sigma_{t_k}(\xbm_{t_k+1}).\\
   \end{aligned}
  \end{equation} 
    The bounds in~\eqref{eqn:EI-convg-1} leads to
    \begin{equation} \label{eqn:EI-convg-1-bound}
  \centering
  \begin{aligned}
               C_4^{4.2} =   \frac{\tau(\sqrt{\beta_{4.2}})}{\tau(-\sqrt{\beta_{4.2}} )},\ C_5^{4.2}= \frac{\tau(\sqrt{\beta_{4.2}})}{\tau(- \sqrt{\beta_{4.2}})}( \sqrt{\beta_{4.2}}+\phi(0)),\\
   \end{aligned}
  \end{equation} 
  where the superscript indicates the corresponding theorems.
  Let $w_{4.6}=2\log(\frac{9}{2\delta})$.  From~\eqref{eqn:EI-convg-star-1}, we have that 
      \begin{equation} \label{eqn:EI-convg-star-1-bound-2}
  \centering
  \begin{aligned}
       C_4^{4.6} = \frac{1}{\Phi(-\sqrt{w_{4.6}})}, \ C_5^{4.6} = \frac{\sqrt{\beta_{4.6}}}{\Phi(-\sqrt{w_{4.6}})} + \frac{\phi(0)}{\Phi(-\sqrt{w_{4.6}})} + \sqrt{\beta_{4.6}}.
   \end{aligned}
  \end{equation}
  Using $\beta_{4.2}>w_{4.6}$, we have that
      \begin{equation} \label{eqn:EI-convg-star-1-bound-3}
  \centering
  \begin{aligned}
    C_4^{4.2} = \frac{\tau(\sqrt{\beta_{4.2}})}{\tau(-\sqrt{\beta_{4.2}})}>  \frac{\tau(\sqrt{\beta_{4.2}})}{\Phi(-\sqrt{\beta_{4.2}})}> \frac{\sqrt{\beta_{4.2}}}{\Phi(-\sqrt{w_{4.6}})}> C_4^{4.6}, 
   \end{aligned}
  \end{equation}
  where the first inequality uses Lemma~\ref{lem:tauvsPhi}, the second one uses Lemma~\ref{lem:EI}, and the last one uses $\beta_{4.2}>1$. 
  Notice that $w_{4.6}>2\log(\frac{9}{2})>3$.
  Then, using $\delta\in(0,1)$, $w_{4.6}<\beta_{4.2}$, and inequalities from~\eqref{eqn:EI-convg-star-1-bound-3}, we have 
      \begin{equation} \label{eqn:EI-convg-star-1-bound-5}
  \centering
  \begin{aligned}
      C_5^{4.6} =& \frac{1}{\Phi(-\sqrt{w_{4.6}})} (\sqrt{\beta_{4.6}}+\phi(0)+ \sqrt{\beta_{4.6}}\Phi(-\sqrt{w_{4.6}}) )<
          \frac{1.042\sqrt{\beta_{4.6}}+\phi(0)}{\Phi(-\sqrt{\beta_{4.2}})} .\\
   \end{aligned}
  \end{equation}
   It is easy to verify that $1.042\sqrt{\beta_{4.6}} +\phi(0)< \beta_{4.2}+\phi(0)\sqrt{\beta_{4.2}}$ by taking the square on both sides. Hence, by  Lemma~\ref{lem:tauvsPhi} and~\ref{lem:EI}, we have 
      \begin{equation} \label{eqn:EI-convg-star-1-bound-6}
  \centering
  \begin{aligned}
      C_5^{4.6} <& \frac{1}{\Phi(-\sqrt{\beta_{4.2}})} (\beta_{4.2}+\phi(0)\sqrt{\beta_{4.2}})<\frac{\tau(\sqrt{\beta_{4.2}})(\sqrt{\beta_{4.2}}+\phi(0))}{\tau(-\sqrt{\beta_{4.2}})} = C_5^{4.2}.
   \end{aligned}
  \end{equation}
\end{proof}
\begin{remark}\label{remark:c1c2}
   An example of Proposition~\ref{cor:EI-convg-compare} is $\delta=0.1$. Then, $\beta_{4.2}=8.19$, $C_4^{4.2}=4632$ and $C_5^{4.2} = 15103$. Meanwhile, $\beta_{4.6}=9.17$, $C_{1}=345$, and $C_{2}=141$. Thus, $C_4^{4.6}=345$, and $C_5^{4.6}=1187$.
  Additional values of $\delta$ are illustrated in Figure~\ref{fig:coeff}.
\end{remark}

\subsection{Improved error bound for RKHS objectives}\label{se:rkhs}
In this section, we demonstrate how the analysis and techniques from Section~\ref{se:asymp-noise} and~\ref{se:rate} can be applied to RKHS objectives to improve existing best-known error bounds and convergence rates from~\cite{bull2011convergence} in the noiseless case.
Formally, Assumptions~\ref{assp:rkhs},~\ref{assp:constraint} are valid in this section. 
The following lemma can be found in~\cite{chowdhury2017kernelized}.
\begin{lemma}\label{lem:rkhs-f-bound}
  In the noiseless scenario, we have at $\forall \xbm\in C$ and $t\in\Nbb$ that
\begin{equation} \label{eqn:rkhs-f-bound-finite-1}
  \centering
  \begin{aligned}
          |f(\xbm) - \mu_t(\xbm)| \leq B\sigma_t(\xbm).
  \end{aligned}
\end{equation}
\end{lemma}
Next, we state the relationship between $I_t$ and $EI_t(\xbm)$, similar to that of Lemma~\ref{lem:IEIbound}.
\begin{lemma}\label{lem:rkhs-bound-EI}
   At $\forall \xbm\in C,  t\in\Nbb$, the following inequality holds 
 \begin{equation} \label{eqn:rkhs-bound-2}
  \centering
  \begin{aligned}
     |I_t(\xbm) - EI_{t}(\xbm)|\leq (B+1) \sigma_{t}(\xbm).
  \end{aligned}
\end{equation}
\end{lemma}
The proof of Lemma~\ref{lem:rkhs-bound-EI} can be found in Lemma 8 in~\cite{bull2011convergence}
We state the asymptotic convergence theorem under RKHS assumptions from~\cite{bull2011convergence} as a lemma.
 \begin{lemma}\label{theorem:EI-convergence-rkhs-1}
   The error of noiseless GP-EI satisfies
   \begin{equation} \label{eqn:EI-convg-rkhs-2}
  \centering
  \begin{aligned}
             r_t \leq c_{\tau}(B) \left[4\frac{M}{t-2} +(B+\phi(0)) \sigma_{t_k}(\xbm_{t_k+1})\right],\\
   \end{aligned}
  \end{equation} 
 where $c_{\tau}(B)=\frac{\tau(B)}{\tau(-B)}$ and $t_k\in[\frac{t}{2}-1, t]$.
\end{lemma}
  The proof of Lemma~\ref{theorem:EI-convergence-rkhs-1} is the same as that of Theorem~\ref{theorem:EI-convergence-1} with corresponding versions of Lemma~\ref{lem:IEIbound-ratio} and~\ref{lem:fmu} for objectives in RKHS. 
  Readers can also find the proof in Theorem 2 of~\cite{bull2011convergence}.
  The improved error bound is given by the next theorem.
\begin{theorem}\label{theorem:EI-convg-star-rkhs}
   Let $C_1 = \frac{1}{\Phi(-B)}$, $C_2= B + \frac{\phi(0)}{\Phi(-B)}$.
   The GP-EI error bound is charaterized by
  \begin{equation} \label{eqn:EI-convg-star-rkhs-noiseless}
  \centering
  \begin{aligned}
      r_t  \leq 4C_1\frac{M}{t-2}+ (C_1B+C_2) \sigma_{t_k}(\xbm_{t_k+1}), 
   \end{aligned}
  \end{equation} 
  where $t_k\in[\frac{t}{2}-1,t]$. Compared to~\eqref{eqn:EI-convg-rkhs-2}, the error bound is smaller.
  \end{theorem}
\begin{proof}
  The proof of~\eqref{eqn:EI-convg-star-rkhs-noiseless} is similar to that of Theorem~\ref{theorem:EI-convg-star} by using Lemma~\ref{lem:rkhs-f-bound} and~\ref{lem:rkhs-bound-EI} and replacing $\sqrt{\beta}$ with $B$. We note that the probability $\delta=0$. 
  Details of the proof are omitted for brevity.
  
  We now show that the error bound in~\eqref{eqn:EI-convg-star-rkhs-noiseless} is better than that in~\eqref{eqn:EI-convg-rkhs-2}.
  Similar to Proposition~\ref{cor:EI-convg-compare}, the constants for~\eqref{eqn:EI-convg-rkhs-2} are  
  \begin{equation} \label{eqn:EI-convg-star-rkhs-pf-1}
  \centering
  \begin{aligned}
     C_4^{4.15} = \frac{\tau(B)}{\tau(-B)},\ C_5^{4.15} =  \frac{\tau(B)}{\tau(-B)}(B+\phi(0)),
   \end{aligned}
  \end{equation}
  while the constants for~\eqref{eqn:EI-convg-star-rkhs-noiseless} are
   \begin{equation} \label{eqn:EI-convg-star-rkhs-pf-2}
  \centering
  \begin{aligned}
     C_4^{4.16} = \frac{1}{\Phi(-B)},\ C_5^{4.16} =  B+\frac{B+\phi(0)}{\Phi(-B)}.
   \end{aligned}
  \end{equation}
  Now define     
\begin{equation} \label{eqn:EI-convg-star-rkhs-pf-3}
  \centering
  \begin{aligned}
     c_r(B) = \frac{\tau(B)(B+\phi(0))}{\Phi(-B)B+B+\phi(0)} <\frac{\tau(B)\Phi(-B)(B+\phi(0))}{\tau(-B)(\Phi(-B)B+B+\phi(0))}= \frac{C_5^{4.15}}{C_5^{4.16}},
   \end{aligned}
  \end{equation}
    where the inequality uses Lemma~\ref{lem:tauvsPhi}.
   Similarly, $\frac{C_4^{4.15}}{C_4^{4.16}}>c_r(B)$.
   Therefore, 
    \begin{equation} \label{eqn:EI-convg-star-rkhs-pf-4}
  \centering
  \begin{aligned}
     r_t^{4.16} < \frac{1}{c_r(B)} r_t^{4.15}.
   \end{aligned}
  \end{equation}
   Notice that $\frac{1}{c_r(B)} \to 0$ as $B\to\infty$. That is, significant improvement in error bound is achieved particularly for a large $B$.

\end{proof}
\begin{remark}\label{remark:EI-convg-rkhs-compare}
   Compared to~\cite{bull2011convergence}, under the same assumption, Theorem~\ref{theorem:EI-convg-star-rkhs} provides a provably smaller bound (see Proposition~\ref{cor:EI-convg-compare}). 
   The convergence rate from~\cite{bull2011convergence} can also be proven,similar to Theorem~\ref{prop:EI-convg-rate-nonoise}, if assumptions (1)-(4) for $k(\xbm,\xbm')$ in~\cite{bull2011convergence} are satisfied. 
\end{remark}

\section{Illustrations of theory}\label{se:examples}
The effectiveness of GP-EI is well-known and well-documented in literature. Readers can find examples in~\cite{jones1998efficient,frazier2018,jiao2019complete,wang2023multifidelity}, etc. 
In this section, we provide illustrations of key steps and features of the convergence analysis and error bounds.
First, we show the illustration of $\tau(z)$ and $\Phi(z)$ for Lemma~\ref{lem:phi} and~\ref{lem:tauvsPhi} in Figure~\ref{fig:tauandphi}.
\begin{figure}
  \centering
  \includegraphics[width=0.55\textwidth]{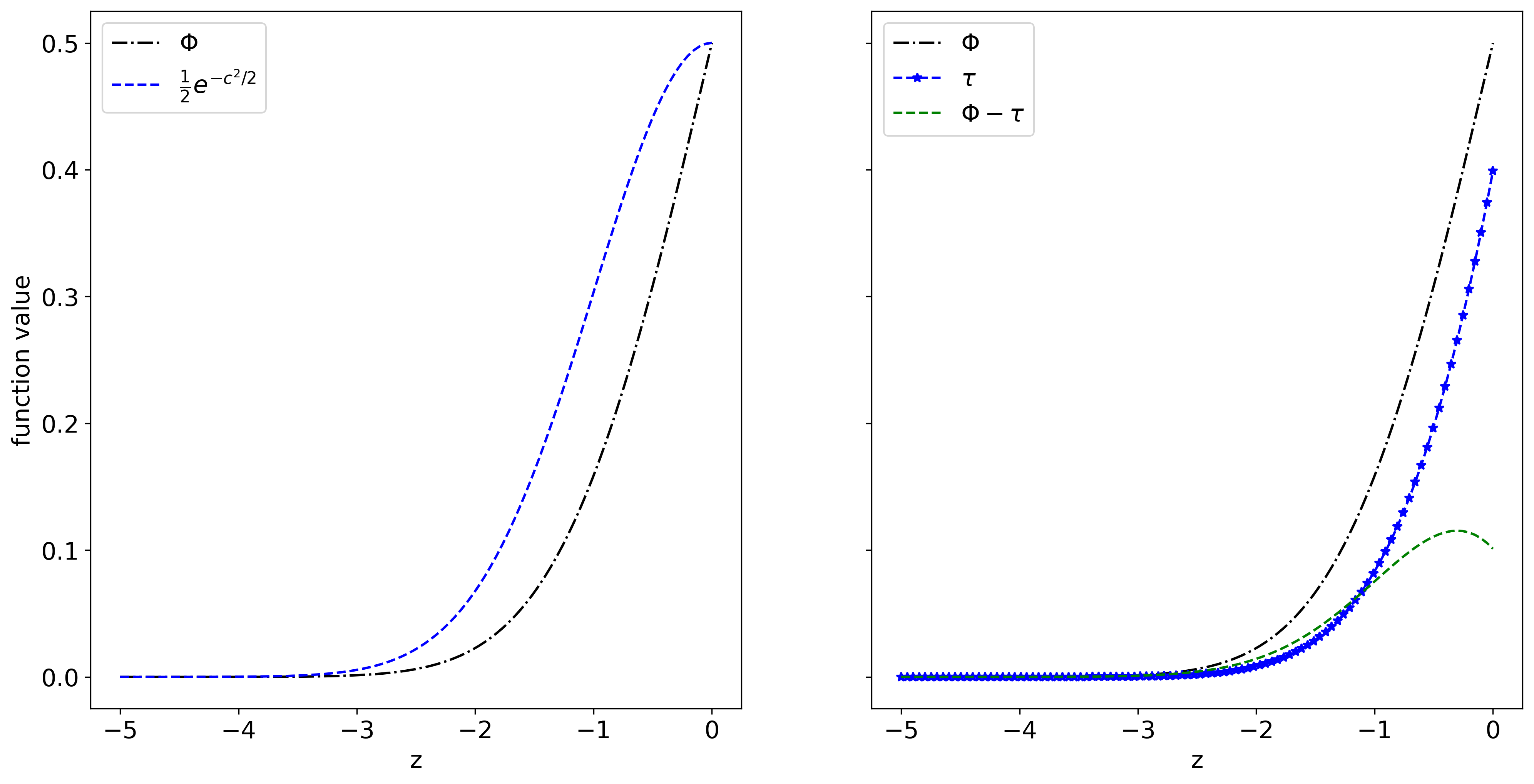}
	\caption{Left: the relationship between $\Phi(c)$ and $\frac{1}{2}e^{-\frac{1}{2}c^2}$ for $c<0$. Right: $\Phi(z)$ \textit{v.s.} $\tau(z)$ for $z<0$.}
\label{fig:tauandphi}
\end{figure}
Next, we plot the contour of EI based on its exploration and exploitation values~\eqref{eqn:EI-ab} in Figure~\ref{fig:contourEI}.
\begin{figure}
  \centering
  \includegraphics[width=0.55\textwidth]{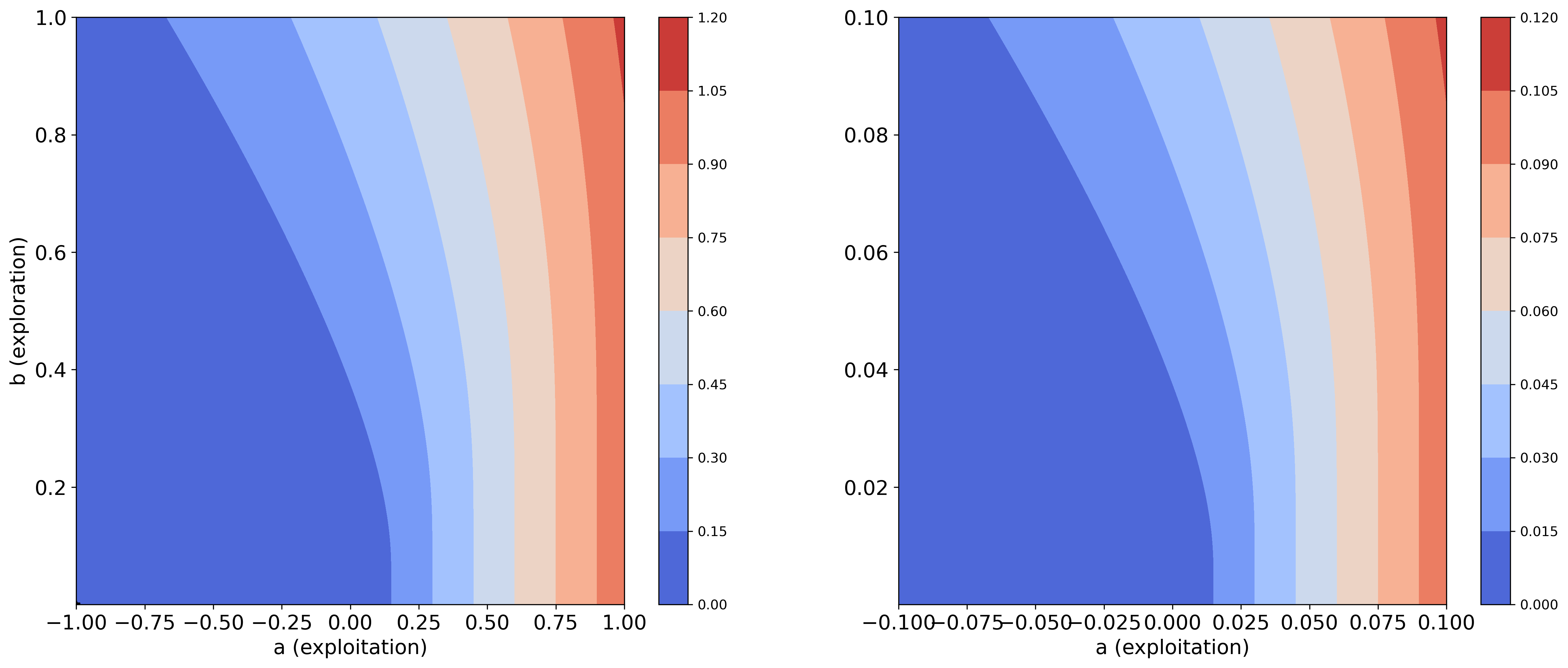}
	\caption{Contour plot for EI using its exploration and exploitation form~\eqref{eqn:EI-ab}.
Zoomed in view of a small exploitation ($a$) is given on the right. It is clear that EI contains intrinsic trade-off between exploration and exploitation.}
\label{fig:contourEI}
\end{figure}

Illustrative examples for functions $\bar{\tau}$ and $\tilde{\tau}$ are given next in Figure~\ref{fig:bartau} and Figure~\ref{fig:tildetau}, respetively. The plots match our analysis in Theorem~\ref{theorem:EI-convg-star}.
\begin{figure}
  \centering
  \includegraphics[width=0.55\textwidth]{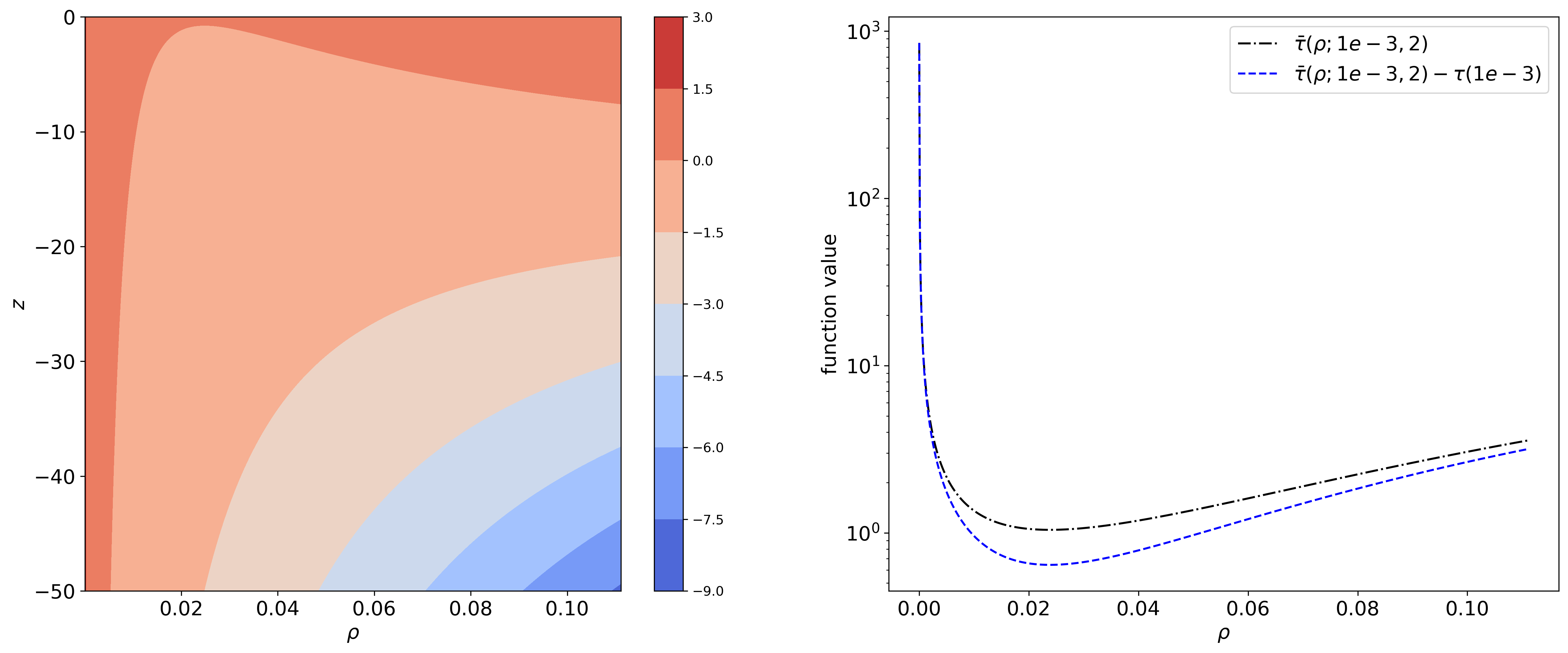}
	\caption{Left: contour plot for $\log_{10}(\bar{\tau})$ with varying $\rho\in (0,\frac{w}{C_3})$ and $z<0$.
Here, $w=2$, $C_1=44$ and $C_3=18$. Right: log-scale comparison of $\bar{\tau}$ with $\tau$ with fixed $z=10^{-3}$. 
  It is clear that $\bar{\tau}(\rho;10^{-3},2,18)-\tau(10^{-3})>0$ and takes a minimum around $\rho=0.02$.}
\label{fig:bartau}
\end{figure}

\begin{figure}
  \centering
  \includegraphics[width=0.55\textwidth]{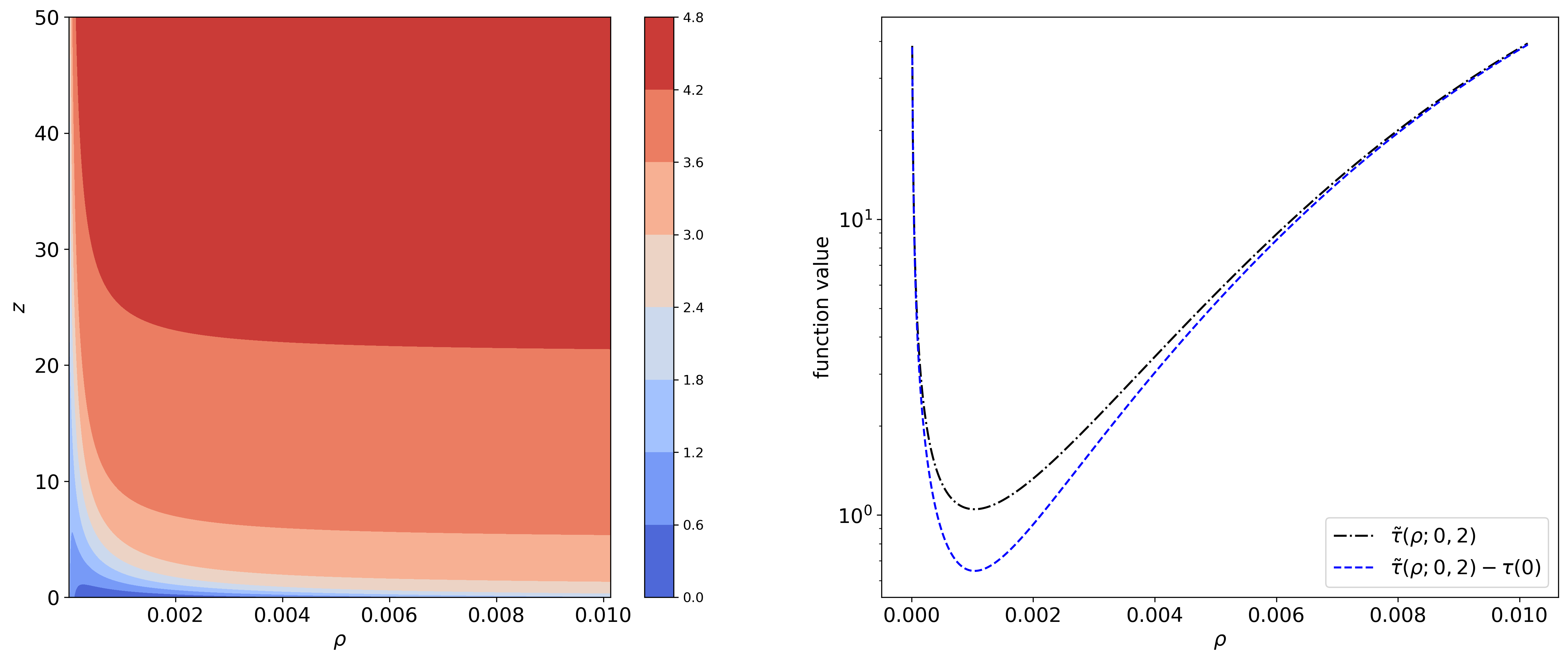}
	\caption{Left: contour plot for $\log_{10}(\tilde{\tau})$ with varying $\rho\in (0,\frac{w}{C_3})$ and $z>0$, which confirms the monotonicity $\tilde{\tau}$ with $z$.
Here, $w=3$, $C_1=741,$ and $C_3=296$. Right: log-scale comparison of $\tilde{\tau}$ with $\tau$ with fixed $z=0$, showing that  
 $\tilde{\tau}(\rho;0,3,741,296)-\tau(0)>0$.}
\label{fig:tildetau}
\end{figure}

Finally, we compare the constant parameters in Proposition~\ref{cor:EI-convg-compare} in Figure~\ref{fig:coeff}.
Given the same probability, it is obvious that the error bound, reflected as the constant parameters, is improved for at least an order of magnitude for most $\delta$ plotted. 
\begin{figure}[H]
  \centering
  \includegraphics[width=0.45\textwidth]{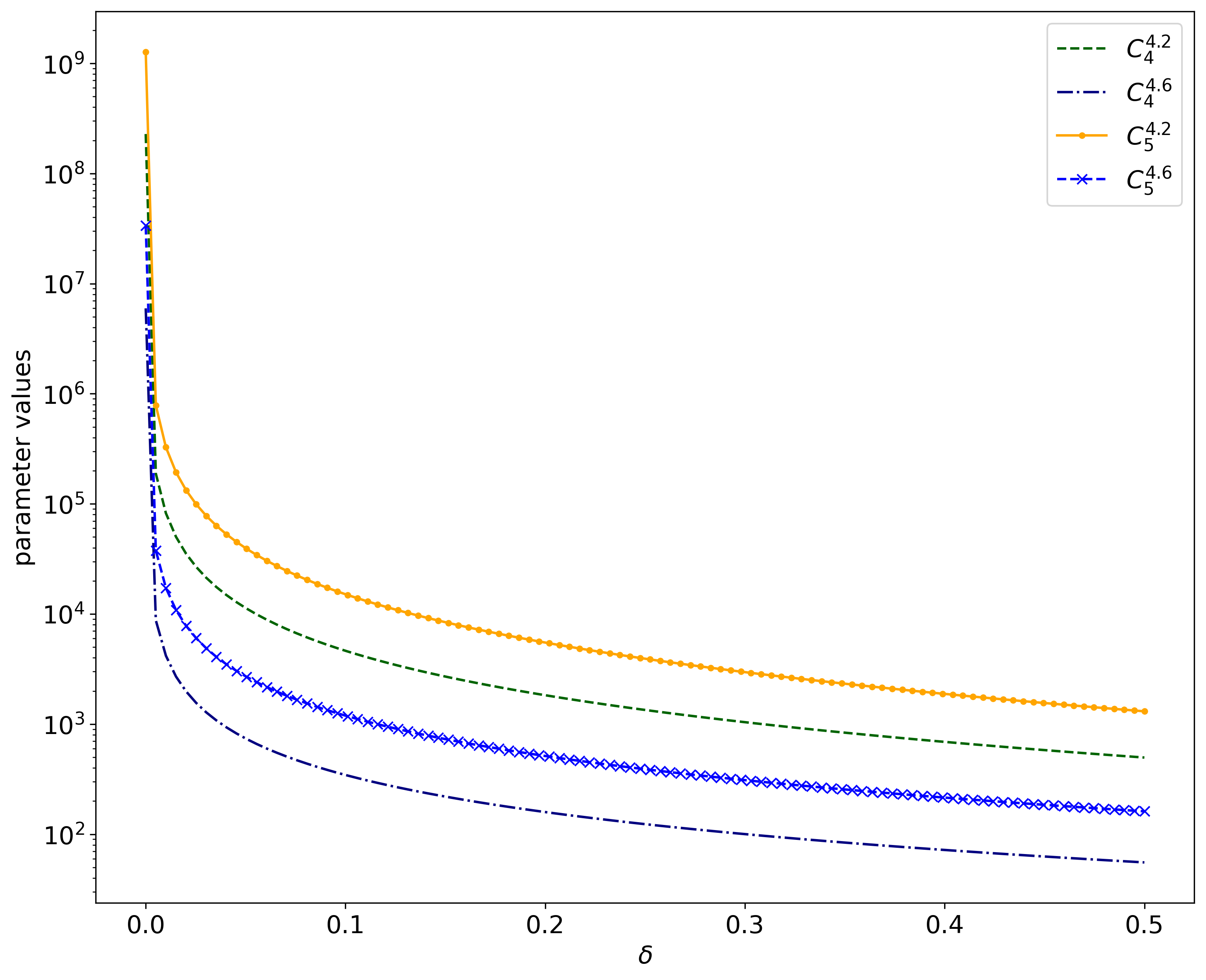}
	\caption{The constant parameters $C_4^{4.2}$, $C_4^{4.6}$, $C_5^{4.2}$, and $C_5^{4.6}$ are plotted in log-scale with respect to $\delta$. It is clear that Theorem~\ref{theorem:EI-convg-star} offers an improved bound, often of at least an order of magnitude.} 
\label{fig:coeff}
\end{figure}

\section{Conclusions}\label{se:conclusion}
This paper addresses gaps in the convergence theory of EI, one of the most widely used BO algorithms. 
We extend the asymptotic convergence results of~\cite{bull2011convergence} to the Bayesian setting where $f$ is sampled from a GP prior.
Further, we establish for the first time the asymptotic convergence results for GP-EI in the noisy case under the GP prior assumption for $f$. 
Last but not least, we use the exploration and exploitation properties introduced here to improve the previously best-known error bounds.
The theoretical insights can guide the design of new EI-based algorithms. 
For instance, given the importance of the standard deviation at the optimal point in our analysis, 
one modification of EI is to add explicit control of the global standard deviation, similar to upper confidence bound (UCB).
Such a modification could lead to improvement in the cumulative regret behavior of GP-EI, a topic for future research.

\section*{Acknowledgments}
This work was performed under the auspices of the U.S. Department of Energy by Lawrence Livermore National Laboratory under contract DE--AC52--07NA27344.  Release number LLNL-JRNL-870576. 

\medskip
\clearpage
\bibliography{bibliography}
\bibliographystyle{unsrt}


\end{document}